\documentclass{article}

\usepackage{style}
\usepackage[utf8]{inputenc} % allow utf-8 input
\usepackage[T1]{fontenc}    % use 8-bit T1 fonts
\usepackage{booktabs}       % professional-quality tables
\usepackage{nicefrac}       % compact symbols for 1/2, etc.
\usepackage{microtype}      % microtypography
\usepackage{url}
\RequirePackage[loading]{tracefnt}

\usepackage{amsmath}
\usepackage{subcaption}
\usepackage{amsfonts}
\usepackage{mathtools}
\usepackage{mathrsfs,bbm}
\usepackage{algorithm}
\usepackage{algorithmic}
\usepackage{comment}
\usepackage{xcolor}
\usepackage[shortlabels]{enumitem}
\usepackage{tikz}
\usetikzlibrary{arrows.meta,positioning,fit,calc}
\usepackage{xcolor}
\usepackage{wrapfig}

\usepackage[textsize=tiny,textwidth=2.0cm]{todonotes}
\newcommand{\up}{^\prime}

\newcommand{\ust}{^\star}

\newcommand{\bE}{\mathbb{E}}
\newcommand{\bN}{\mathbb{N}}
\newcommand{\bP}{\mathbb{P}}

\newcommand{\bR}{\mathbb{R}}

\newcommand{\cA}{\mathcal{A}}

\newcommand{\cF}{\mathcal{F}}
\newcommand{\cG}{\mathcal{G}}

\newcommand{\cO}{\mathcal{O}}
\newcommand{\ctO}{\tilde{\mathcal{O}}}
\newcommand{\cS}{\mathcal{S}}
\newcommand{\cT}{\mathcal{T}}
\newcommand{\cX}{\mathcal{X}}

\newcommand{\flbr}[1]{\left\{#1\right\}}
\newcommand{\sqbr}[1]{\left[#1\right]}
\newcommand{\br}[1]{\left(#1\right)}

\newcommand{\abs}[1]{\left|#1\right|}
\newcommand{\norm}[1]{\left\Vert #1 \right\Vert}

\newtheorem{assumption}{Assumption}
\tikzset{
    state/.style={circle, draw, thick, minimum size=7mm, inner sep=0pt},
    goalstate/.style={circle, draw, thick, fill=yellow!30, minimum size=7mm, inner sep=0pt},
    obslabel/.style={font=\scriptsize},
    wall/.style={fill=black!20, draw=black, thick},
    cell/.style={draw=black, thick, minimum width=8mm, minimum height=8mm},
    rock/.style={circle, draw=black, fill=gray!25, minimum size=5mm, inner sep=0pt},
    goodrock/.style={circle, draw=black, fill=green!30, minimum size=5mm, inner sep=0pt},
    badrock/.style={circle, draw=black, fill=red!25, minimum size=5mm, inner sep=0pt},
}

\makeatletter
\@ifundefined{keywords}{}{}
\makeatother

\title{Policy Gradient Methods for Non-Markovian Reinforcement Learning}
\editor{}
\author{Avik Kar\thanks{Department of Computer Science and Automation, Indian Institute of Science, Bengaluru, India. \email{avikkar@iisc.ac.in}}
, Siddharth Chandak\thanks{Equal contribution.}~\thanks{Department of Electrical Engineering, Stanford University, Stanford, CA, USA.}
, Rahul Singh\footnotemark[2]
, Soumitra Sinhahajari\thanks{Department of Electrical and Electronics Engineering, Nanyang Technological University, Singapore.}
, Eric Moulines\thanks{CMAP, CNRS, \'Ecole polytechnique, Institut Polytechnique \'de Paris, Palaiseau, France.}
, Shalabh Bhatnagar\footnotemark[1]
\and Nicholas Bambos \footnotemark[3]
}
\begin{document}

\maketitle

\begin{abstract}
We study policy gradient methods for reinforcement learning in non-Markovian decision processes (NMDPs), where observations and rewards depend on the entire interaction history. To handle this dependence, the agent maintains an internal state that is recursively updated to provide a compact summary of past observations and actions. In contrast to approaches that treat the agent state dynamics as fixed or learn it via predictive objectives, we propose a reward-centric formulation that jointly optimizes the agent state dynamics and the control policy to maximize the expected cumulative reward. To this end, we consider a class of Agent State-Markov (ASM) policies, comprising an agent state dynamics and a control policy that maps the agent state to actions. We establish a novel policy gradient theorem for ASM policies, extending the classical policy gradient results from the Markovian setting to episodic and infinite-horizon discounted NMDPs. Building on this gradient expression, we propose the Agent State-Markov Policy Gradient (ASMPG) algorithm, which leverages the recursive structure of the agent state dynamics for efficient optimization. We establish finite-time and almost sure convergence guarantees, and empirically demonstrate that, on a range of non-Markovian tasks, ASMPG outperforms baselines that learn state representations via predictive objectives.
\end{abstract}

% \begin{keywords}
%   Reinforcement learning, Non-Markovian decision processes, Policy gradient.
% \end{keywords}

\section{Introduction}
\label{sec:intro}

Reinforcement Learning (RL) provides a general framework for sequential decision-making and has enabled major advances across domains ranging from large language models (LLMs)~\citep{ouyang2022training} to robotics~\citep{levine2016end} and healthcare~\citep{yu2021reinforcement}. A large body of work  models the environment as a Markov decision process (MDP), where the dynamics and rewards depend only on the current state and action, and not on the past history. This assumption enables tractable analysis and algorithm design. However, in many real-world settings, the Markov property fails to hold, as the dynamics and rewards depend on the entire history of observations and actions~\citep{de2009inventory,pang2023condition}. Such environments are inherently ``non-Markovian'' and cannot be accurately captured by standard MDP formulations.

A canonical example of non-Markovianity arises in partially observable environments, where the agent must infer latent states from a history of observations to act optimally~\citep{kaelbling1998planning}. An intuitive instance appears in dialogue systems~\citep{li2016deep}, where appropriate responses depend on the full conversation history rather than the most recent utterance alone. Similar challenges arise in domains such as recommender systems~\citep{wu2021partially}, financial markets~\citep{liu2022finrl}, and healthcare~\citep{strobl2024modulate}, where the underlying dynamics depend on latent factors or long-term dependencies. In such settings, effective control requires extracting and maintaining a representation of relevant information from the trajectory.

Learning such representations, however, poses several challenges. While the full history is sufficient, its size grows with time, rendering direct use of history computationally intractable. Moreover, representation learning is tightly coupled with control, requiring the agent to jointly learn both a state representation and a policy. A common approach in non-Markovian RL is to learn such representations by optimizing them to predict future observations based on the current state \citep{subramanian2022approximate, chandak2024reinforcement}. These predictive representations aim to summarize the history in a way that is sufficient for modeling the environmental dynamics. However, learning such predictive representations can be challenging, as accurate prediction may require a high-dimensional state representation and can be difficult to learn in practice.

In this work, we take a different approach, motivated by the view that predicting future observations is merely an intermediate objective. Instead, we bypass this step and directly learn a state representation and policy optimized for maximizing reward, the ultimate objective in reinforcement learning. To this end, we introduce a class of Agent State-Markov (ASM) policies, where the agent follows a Markov policy with respect to an agent state that summarizes the history and is recursively updated. We extend the policy gradient theorem to this class of policies, which enables us to develop the \textbf{Agent State-Markov Policy Gradient (ASMPG)} algorithm for both episodic and infinite-horizon discounted non-Markovian decision processes (NMDPs). Empirically, we demonstrate that in various non-Markovian environments, ASMPG outperforms methods that learn state representations via predictive objectives.

\subsection{Main contributions} 
We summarize the main contributions of this work below.

\textbf{Reward-centric optimization of agent state dynamics and control policy.} 
We formulate policy search over the joint parameters of the recursive agent-state dynamics and the state-dependent control policy. In contrast to existing approaches that treat the agent state dynamics as fixed \citep{dong2022simple} or learn it via auxiliary predictive objectives \citep{subramanian2022approximate}, we propose a novel reward-centric formulation that departs from predictive objectives and directly optimizes the agent state dynamics and the control policy for the control objective.

\textbf{Policy gradient theorem for NMDPs and ASMPG algorithm.} 
To enable joint optimization, we establish a novel policy gradient theorem for non-Markovian RL, extending classical policy gradient results beyond the Markovian setting to both episodic and infinite-horizon discounted NMDPs. The resulting gradient admits a simple form that exploits the recursive structure of the agent state dynamics, depending on the history only through the previous state. This leads to the ASMPG algorithm, an efficient method for jointly learning the agent state dynamics and control policy.

\textbf{Theoretical guarantees and empirical performance.} 
Under standard assumptions, we establish a finite-time convergence rate of $O(1/\sqrt{K})$ for the time average of the squared gradient norm, together with almost-sure convergence of the gradient norm to zero under decreasing stepsizes. We evaluate ASMPG on five non-Markovian environments and compare it with the AIS-KL and AIS-MMD baselines of~\citet{subramanian2022approximate} under comparable representation sizes. Across these tasks, ASMPG performs favorably, supporting the benefit of a reward-centric formulation.

\subsection{Related work}
\label{sec:lit_survey}

Reinforcement learning in non-Markovian environments is largely organized around how the growing history is compressed. A classical case is partial observability: in a POMDP, the belief over latent states is a sufficient statistic for optimal control when the model is known or accurately estimated~\citep{astrom1965optimal,smallwood1973optimal,kaelbling1998planning,poupart2011closing}. Without such a model, common alternatives include finite observation windows~\citep{loch1998using}, finite-state controllers~\citep{poupart2003bounded}, predictive state representations~\citep{littman2001predictive,singh2003learning}, and recurrent history encoders~\citep{schmidhuber1990reinforcement,bakker2001reinforcement,ni2022recurrent}. These methods reduce history dependence, but typically rely on a fixed representation class or provide limited guarantees for end-to-end reward optimization.

More structured approaches explicitly define agent-side states. \citet{dong2022simple} study recursively updated agent states, but assume the state dynamics are fixed and optimize only the control policy. \citet{sinha2024periodic} take a related fixed agent state approach for POMDPs. Approximate information states~\citep{subramanian2022approximate} and related analyses~\citep{chandak2024reinforcement} provide guarantees when the representation is approximately sufficient for rewards and future prediction. Orthogonal automata-based approaches, such as regular decision processes and reward machines, expose temporal structure through finite-state abstractions~\citep{brafman2019regular, abadi2020learning, icarte2018using}. In contrast, our formulation is reward-centric: the recursive agent-state dynamics and the control policy are optimized jointly for cumulative reward.

The closest literature is policy-gradient learning with memory. GPOMDP estimates gradients for POMDPs without observing the latent state~\citep{baxter2001infinite}; finite-state controllers and structured POMDP policies have been optimized with gradient or actor-critic methods~\citep{meuleau1999learning,yu2005function}; sufficient statistics such as beliefs and PSRs have been combined with policy gradients~\citep{aberdeen2007policy}; and recurrent policy-gradient or natural actor-critic methods optimize RNN-based memory policies~\citep{wierstra2010recurrent, cayci2024finite, cayci2025recurrent}. Policy gradient method for MDPs is now well understood~\citep{sutton1999policy, agarwal2021theory, zhang2021sample}. Our contribution is to derive an ASM policy-gradient theorem for general NMDPs and to use it for joint optimization of the agent-state dynamics and the control policy, rather than optimizing only a fixed memory, a model-based sufficient statistic, or a generic recurrent policy.

% Our contribution is to derive an ASM policy-gradient theorem for general NMDPs and to use it for joint optimization of the agent-state dynamics and the control policy, rather than optimizing a control policy with fixed agent-state dynamics, a model-based sufficient statistic, or a generic recurrent controller.

\textbf{Notation}. $\bN$ denotes the set of natural numbers, and $[N]$ denotes the set consisting of the first $N$ natural numbers. For a set $\cX$, $\Delta_{\cX}$ denotes the space of probability distributions on $\cX$. In general, we use capital letters to denote random variables and small letters to denote their realizations.

\section{Episodic non-Markovian decision process}
\label{subsec:fhnmdp}
An episodic NMDP is specified by a five-tuple $(\mathcal O, \mathcal A, \{p_t\}_{t=1}^{H-1}, \{r_t\}_{t=1}^H, H)$, where $\mathcal O$ and $\mathcal A$ denote the finite observation and action spaces, $H$ is the episode duration, and $\{p_t\}_{t=1}^H$ and $\{r_t\}_{t=1}^H$ are the sequence of transition kernels and reward functions. The initial observation is sampled as $O_1\sim \mu$. At each step $t\in[H]$, the agent observes $O_t\in\cO$ and then chooses $A_t\in\cA$. We write $O_{1:t}\coloneqq (O_1,\ldots,O_t)$ and $A_{1:t}\coloneqq (A_1,\ldots,A_t)$.  The process is non-Markovian because the next observation and the reward function depend on the entire history until the current time-step, i.e., for all $t \in [H-1]$, $O_{t+1} \sim p_t(\cdot \mid O_{1:t}, A_{1:t})$, and the agent receives a reward $r_t(O_{1:t}, A_{1:t})$. We assume throughout that rewards are uniformly bounded, i.e., there exists $r_{\max}>0$ such that $\abs{r_t(\cdot)}\le r_{\max}$ for every $t\in[H]$. For $1\le t_1\le t_2\le H$, define the partial return
\begin{align*}
    R_{t_1:t_2} \coloneqq \sum_{t=t_1}^{t_2} r_t(O_{1:t},A_{1:t}),
\end{align*}
with the convention that an empty sum is zero.

Let $\Pi_{\mathrm{HR}}$ denote the class of history-dependent randomized policies. A policy $\rho\in\Pi_{\mathrm{HR}}$ is a sequence $\rho=\{\rho_t\}_{t=1}^H$ with $\rho_t:(\cO\times\cA)^{t-1}\times\cO\to\Delta_{\cA}$, where $\rho_t(\cdot\mid o_{1:t},a_{1:t-1})$ is the action distribution used after observing $(o_{1:t},a_{1:t-1})$. The benchmark objective is
\begin{align*}
    \rho^\star \in \arg\max_{\rho\in\Pi_{\mathrm{HR}}}
    \bE_{\rho,\mu}\sqbr{R_{1:H}},
\end{align*}
where $\bE_{\rho,\mu}$ denotes expectation under the trajectory distribution
induced by $\rho$ and $O_1\sim\mu$.

\subsection{Agent state and ASM policies}
\label{subsec:infostate}
To avoid conditioning directly on the growing history, the agent maintains an \emph{agent state} $S_t\in\cS$ at each time $t$, where $\cS$ is a finite state space~\citep{dong2022simple}. The state is updated recursively according to a stochastic kernel
\begin{align*}
    \nu_t : \cS\times\cA\times\cO \to \Delta_{\cS},
    \qquad
    S_t \sim \nu_t(\cdot\mid S_{t-1},A_{t-1},O_t),
    \quad t\in[H],
\end{align*}
with initialization $(S_0,A_0)=(s_0,a_0)$, where $(s_0,a_0)$ is a fixed dummy state-action pair used only to start the recursion. We write $\nu=\{\nu_t\}_{t=1}^H$ and call $\nu$ the \textbf{agent state dynamics}. A deterministic update is included as the special case in which $\nu_t(\cdot\mid s,a,o)$ is a Dirac mass. Through the recursion, $S_t$ is a summary of the interaction history available before choosing $A_t$.

Given an agent state, the control policy is Markovian in $S_t$. Specifically, let $\phi=\{\phi_t:\cS\to\Delta_{\cA}\}_{t=1}^H$ and sample
\begin{align*}
    A_t \sim \phi_t(\cdot\mid S_t),
    \quad t\in[H].
\end{align*}
The pair $(\nu,\phi)$ induces a history-dependent policy, denoted $\phi\circ\nu$. We refer to such policies as \textbf{Agent State-Markov (ASM) policies}.

\begin{wrapfigure}{!t}{0.55\textwidth}
\centering
\resizebox{\linewidth}{!}{\begin{tikzpicture}[
    >=Latex,
    x=0.88cm,
    y=0.80cm,
    font=\normalsize,
    obs/.style={draw, rounded corners=2pt, thick, fill=blue!18, minimum width=3.8cm, minimum height=1.45cm, align=center},
    act/.style={draw, rounded corners=2pt, thick, fill=green!20, minimum width=3.5cm, minimum height=1.2cm, align=center},
    st/.style={draw, rounded corners=2pt, thick, fill=orange!20, minimum width=4.0cm, minimum height=1.2cm, align=center},
    upd/.style={draw, rounded corners=2pt, thick, draw=violet!80!black, fill=violet!18, minimum width=3.8cm, minimum height=1.6cm, align=center},
    envbox/.style={draw=blue!80!black, rounded corners=6pt, thick, inner sep=10pt},
    agentbox/.style={draw=orange!85!black, rounded corners=6pt, thick, inner sep=8pt},
    lab/.style={font=\normalsize, align=center}
]

% Loop nodes
\node[act] (A) at (-3.6,0) {Chatbot response\\[1mm] $A_t \sim \phi_t(\cdot \mid S_t)$};
\node[upd] (U) at (3.6,0) {State update, $\nu_t$\\[1mm] $S_{t+1} \sim \nu_{t+1}(S_t, A_t, O_{t+1})$};
\node[obs] (O) at (0,2.75) {Next user utterance\\[1mm] $O_{t+1} \sim p_t(\cdot \mid O_{1:t}, A_{1:t})$};

% Group boxes
\node[blue!80!black, font=\bfseries\normalsize] at ($(O.north)+(0,0.5)$) {Non-Markovian environment};
\coordinate (AGENTLOOPL) at ($ (A.south west) + (-0.1,-1.2) $);
\coordinate (AGENTLOOPR) at ($ (U.south east) + (0.1,-1.2) $);
\node[agentbox, fit=(A)(U)(AGENTLOOPL)(AGENTLOOPR),
      label={[orange!85!black, yshift=-6pt]below:\bfseries\normalsize Agent}] (AGENTLOOP) {};

% Loop arrows
\draw[->, thick]
    ($(A.east)+(0,0.18)$) -- node[above, lab] {$A_t$} ($(U.west)+(0,0.18)$);

\draw[->, thick, rounded corners=10pt]
    (A.north) |- (O.west);

\draw[->, thick, rounded corners=10pt]
    (O.east) -- ++(0.45,0) -| node[pos=0.72, right, lab] {$O_{t+1}$} (U.north);

\draw[->, thick, rounded corners=10pt]
    (U.south) -- ++(0,-1.15) -| node[pos=0.75, left, lab] {$S_t$} (A.south);
\node[font=\normalsize, align=center, fill=white, inner sep=1pt] at (-1.9,-1.35) {Dialogue summary\\state};

\draw[->, thick, rounded corners=10pt]
    (U.south) -- ++(0,-1.15) -- ++(-3.0,0) -- node[pos=0.5, left, lab] {$S_t$} ++(0,1.85) -- ($(U.west)+(0,-0.38)$);
\end{tikzpicture}}
\caption{Chatbot as a non-Markovian environment with agent state.}
\label{fig:dialogue_asd}
\end{wrapfigure}
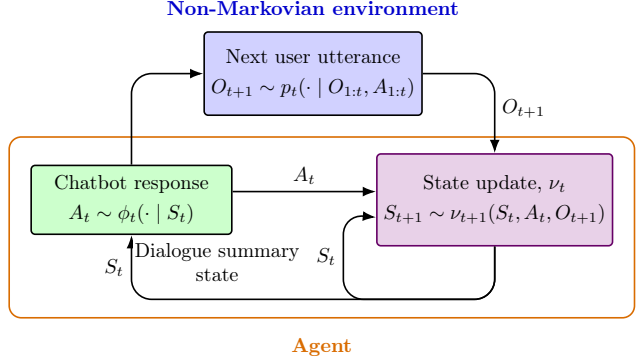

We illustrate the role of agent state and ASM policies using a chatbot~\citep{young2013pomdp}. At each time step $t$, the user (which is the environment in this example) utters $O_t \in \cO$. The agent selects an action $A_t \in \cA$, a response to the user, and receives a reward $r_t$ based on the user's satisfaction.~The environment is non-Markovian, as both the utterance and the reward depend on the entire conversation history. For example, a response such as ``yes'' can only be interpreted in the context of previous interactions.~To act effectively, the agent maintains a structured state $S_t \in \cS$ that summarizes relevant information from the chat history, such as the user's intent, previously mentioned entities, or the status of the dialogue. The state is updated recursively as $S_t \sim \nu_t(\cdot \mid S_{t-1}, A_{t-1}, O_t)$, where $\nu_t$ extracts information from the latest utterance and updates the state. The agent then selects actions based only on the current state, $A_t \sim \phi_t(\cdot \mid S_t)$. Figure~\ref{fig:dialogue_asd} illustrates this interaction.

\begin{remark}
    \citet{dong2022simple} consider a related notion of agent state-based Markov policies, but assume fixed agent state dynamics and optimize only the control policy. The approximate information state (AIS) framework of \citet{subramanian2022approximate} is defined with respect to predicting future observations; generalizing this notion to arbitrary objectives leads to the class of ASM policies. Thus, the policy classes in both works are subsumed by the class of ASM policies.
\end{remark}

We optimize directly over ASM policies:
\begin{align}
    (\phi^\star,\nu^\star) \in
    \arg\max_{\phi,\nu} \bE_{\phi\circ\nu,\mu}\sqbr{R_{1:H}}.
    \label{def:best_asm}
\end{align}
This restriction is not intrinsically lossy when the class of agent state dynamics is rich enough. In particular, if $\cS$ is allowed to contain all histories $(O_{1:t},A_{1:t-1})$, then the recursion can store the complete history and an ASM policy can emulate any policy in $\Pi_{\mathrm{HR}}$. In Appendix~\ref{app:asd_optimality}, we show that if an ``ideal'' agent state dynamics exists, then an ASM policy that uses it is optimal. In practice, the goal is to learn a much smaller state space that is sufficient for high reward.

\subsection{ASM policy gradient for episodic NMDPs}
\label{subsec:fhnmpg}
We now derive a policy-gradient identity for parameterized ASM policies. It is convenient to combine the agent state dynamics and control policy into a single kernel $\pi_t : \cS\times\cA\times\cO \to \Delta_{\cS\times\cA}$, where $\pi_t(s,a\mid \tilde{s},\tilde{a},o)$ is the conditional distribution of $(S_t,A_t)$ given $(S_{t-1},A_{t-1},O_t)=(\tilde{s},\tilde{a},o)$. Thus, $\pi_t$ can be factorized in the following manner.
\begin{align}\label{eqn-factors}
    \pi_t(s,a\mid \tilde{s},\tilde{a},o)
    = \nu_t(s\mid \tilde{s},\tilde{a},o)\,\phi_t(a\mid s),
    \quad
    s,\tilde{s}\in\cS,\ a,\tilde{a}\in\cA,\ o\in\cO.
\end{align}

Consider a differentiable parametric class
$\pi_\theta=\{\pi_{\theta,t}\}_{t=1}^H$, with $\theta\in\bR^d$, and define 
\begin{align*}
    J(\theta) \coloneqq \bE_{\pi_\theta,\mu}\sqbr{R_{1:H}},
\end{align*}
which is the expected sum of rewards over an episode that we wish to maximize over $\theta\in\bR^d$. We now state the corresponding policy gradient theorem. 
\begin{theorem}[ASM Policy Gradient Theorem for Episodic NMDPs] \label{thm:fhpg}
    Let $\{\pi_\theta = \{\pi_{\theta,t}\}_{t=1}^H \mid \theta \in \bR^d\}$ be a class of parameterized episodic ASM policies that are differentiable with respect to $\theta$, and have support independent of $\theta$ for all $t \in [H]$. Then,
    \begin{align}
        \nabla_\theta J(\theta)
        = \bE_{\pi_\theta,\mu}\sqbr{
            \sum_{t=1}^{H}
            R_{t:H}\,
            \nabla_\theta \log \pi_{\theta,t}(S_t,A_t\mid S_{t-1},A_{t-1},O_t)}.
        \label{eq:fhpg_return_form}
    \end{align}
\end{theorem}
See Appendix~\ref{subapp:pg_episodic} for the proof. The above theorem implies that from a single trajectory generated by $\pi_\theta$, the agent can compute 
\begin{align*}
    \sum_{t=1}^{H}
    R_{t:H}\,
    \nabla_\theta \log \pi_{\theta,t}(S_t,A_t\mid S_{t-1},A_{t-1},O_t),
\end{align*}
which is the unbiased stochastic gradient estimator of $\nabla_\theta J(\theta)$. 
This reward-to-go form is similar to REINFORCE~\citep{williams1992simple}, where the rewards earned before time $t$ do not contribute to the conditional score at time $t$. This yields the episodic Agent State-Markov Policy Gradient algorithm (Algorithm~\ref{algo:fhpg}), an online policy gradient algorithm for episodic NMDPs. Due to the factorizability property of ASM policies~\eqref{eqn-factors}, we can decompose the gradient in the following form.
\begin{align*}
    \nabla_\theta \log \pi_{\theta,t}(S_t, A_t \mid S_{t-1}, A_{t-1}, O_t) = \nabla_\theta \log \nu_{\theta,t}(S_t \mid S_{t-1}, A_{t-1}, O_t) + \nabla_\theta \log \phi_{\theta,t}(A_t \mid S_t).
\end{align*}
This enables us to use decoupled kernels for the agent state
dynamics and the control policy, and update both of them separately.

\begin{algorithm}[t]
    \caption{Agent State-Markov Policy Gradient (ASMPG) for Episodic NMDPs}
    \label{algo:fhpg}
    \begin{algorithmic}
        \STATE {\bfseries Input} horizon $H$, stepsizes $\{\alpha_k\}_{k\ge 1}$, initial parameter $\theta_1$
        \FOR{$k=1,2,\ldots$}
            \STATE Play $\pi_{\theta_k}$ for one episode and collect $(o^k_{1:H},s^k_{0:H},a^k_{0:H})$, the trajectory of the $k$-th episode
            \STATE Compute
            \vspace{-4pt}
            \begin{align*}
                \widehat g_k
                = \sum_{t=1}^{H}
                \br{\sum_{\ell=t}^{H} r_\ell(o^k_{1:\ell},a^k_{1:\ell})}
                \nabla_\theta \log \pi_{\theta_k,t}(s^k_t,a^k_t\mid s^k_{t-1},a^k_{t-1},o^k_t).
            \end{align*}
            \vspace{-12pt}
            \STATE Update $\theta_{k+1}=\theta_k+\alpha_k\widehat g_k$
        \ENDFOR
    \end{algorithmic}
\end{algorithm}

\subsubsection{Convergence guarantees}
We next state non-asymptotic convergence guarantees for Algorithm~\ref{algo:fhpg}.
The analysis relies on the following uniform score and Hessian bounds.

\begin{assumption}
\label{assum:smooth_parame_fh}
    The parametrized ASM kernels $\{\pi_{\theta,t}\}_{t=1}^H$ satisfy the
    following property: there exist constants $G,M>0$ such that, for every
    $t\in[H]$, $s,\tilde{s}\in\cS$, $a,\tilde{a}\in\cA$, $o\in\cO$, and
    $\theta\in\bR^d$,
    \begin{align*}
        \norm{\nabla_\theta \log \pi_{\theta,t}(s,a\mid \tilde{s},\tilde{a},o)}_2 \le G,
        \qquad
        \norm{\nabla_\theta^2 \log \pi_{\theta,t}(s,a\mid \tilde{s},\tilde{a},o)}_2 \le M.
    \end{align*}
\end{assumption}

\begin{remark}
    Assumption~\ref{assum:smooth_parame_fh} is standard in finite-time analyses of policy-gradient methods. For example, Lemma~\ref{lem:softmax_asm} in Appendix~\ref{app:aux} verifies the assumption for a tabular softmax parametrization of the combined kernel $\pi_{\theta,t}$, with $G=\sqrt{2}$ and $M=1$. If the combined kernel is instead implemented by separate tabular softmax parametrizations of $\nu_t$ and $\phi_t$, the same assumption holds with different universal constants.
\end{remark}

The next lemma gives a uniform smoothness bound for the non-Markovian objective. 
\begin{lemma}[Smoothness of the episodic ASM objective]\label{lem:finite_smooth}
    Suppose that the setting for Theorem~\ref{thm:fhpg} and Assumption~\ref{assum:smooth_parame_fh} hold. Then $J$ is $\beta$-smooth, i.e., $\left\|\nabla_\theta^2 J(\theta)\right\|_2 \le \beta$ for all $\theta\in\mathbb{R}^d$, where
    \[
        \beta
        :=
        \frac{r_{\max}H(H+1)}{6}
        \left(3M+G^2(2H+1)\right).
    \]
\end{lemma}
Combining Lemma~\ref{lem:finite_smooth} with the boundedness of the stochastic gradient estimator gives the following high-probability convergence bound.
\begin{theorem}[Finite-time convergence of ASMPG]\label{thm:conv_fhpg}
     Consider the iterates $\{\theta_k\}_{k\ge 1}$ generated by Algorithm~\ref{algo:fhpg}. Suppose that the setting for Theorem~\ref{thm:fhpg} and Assumption~\ref{assum:smooth_parame_fh} hold. Let $\Delta_1\coloneqq \sup_{\theta\in\bR^d} J(\theta)-J(\theta_1)$, let $\beta$ be as in Lemma~\ref{lem:finite_smooth}, and set $C_H\coloneqq r_{\max}G H(H+1)$. Fix $K\ge 1$ and $\delta\in(0,1)$. If
    \begin{align*}
        \alpha_k = \min\flbr{\frac{1}{\beta},\frac{1}{C_H}\sqrt{\frac{\Delta_1}{\beta K}}}, \quad k=1,\ldots,K,
    \end{align*}
    then, with probability at least $1-\delta$,
    \begin{align*}
        \frac{1}{K}\sum_{k=1}^{K}\norm{\nabla_\theta J(\theta_k)}_2^2 \leq \frac{2\beta\Delta_1}{K} + 5 C_H\sqrt{\frac{\beta\Delta_1}{K}} + \frac{12C_H^2\log(1/\delta)}{K}.
    \end{align*}
\end{theorem}

The proofs of Lemma~\ref{lem:finite_smooth} and Theorem~\ref{thm:conv_fhpg} are deferred to Appendix~\ref{subapp:smooth_episodic} and Appendix~\ref{subapp:proof_conv_episodic}, respectively. Additional in-expectation bounds, as well as finite-time and almost sure convergence results under decreasing stepsizes, are provided in Appendix~\ref{subapp:proof_conv_episodic}. For fixed $r_{\max}$, $G$, $M$, and $\Delta_1$, the leading term in Theorem~\ref{thm:conv_fhpg} scales as $\ctO(H^{3.5}/\sqrt{K})$ for the time-averaged squared gradient norm.
\section{Infinite-horizon discounted non-Markovian decision process}
\label{sec:ihdrnmdp}

An infinite-horizon discounted NMDP is described by a five-tuple $(\cO, \cA, \{p_t\}_{t=1}^\infty, \{r_t\}_{t=1}^\infty, \gamma)$, where $\gamma \in (0,1)$ is the discount factor.~As in the finite-horizon case, $O_t \in \cO$ is the observation and $A_t \in \cA$ is the action at step $t$, but now the process evolves indefinitely. The observation process evolves as $O_{t+1} \sim p_t(\cdot \mid O_{1:t}, A_{1:t})$, for $t = 1,2,\ldots$, and at the $t$-th step the agent receives the reward $r_t(O_{1:t},A_{1:t})$. For $1\le t_1\le t_2 < \infty$, define the discounted partial return
\begin{align*}
    R^\gamma_{t_1:t_2} \coloneqq \sum_{t=t_1}^{t_2}{\gamma^{t-t_1} r_t(O_{1:t},A_{1:t})},
\end{align*}
with the convention that an empty sum is zero, and define $R^\gamma_{t_1:\infty} = \lim_{t_2 \to \infty}{R^\gamma_{t_1:t_2}}$. The limit exists due to discounting and the boundedness of the rewards. The benchmark objective is to maximize the expected value of the cumulative discounted rewards $R^\gamma_{1:\infty}$, i.e., solving the following infinite-horizon discounted optimization problem,
\begin{align*}
    \pi^\star \in \arg\max_{\pi \in \Pi_{\text{HR}}} \bE_{\pi,\mu}\sqbr{R^\gamma_{1:\infty}},
\end{align*}
where once again $\Pi_{\text{HR}}$ is the class of history-dependent policies, and a policy from the set $\Pi_{\text{HR}}$ at each step $t$ prescribes an action as a function of the entire history until $t$. Note that, in this case, employing a general history-dependent policy is practically impossible due to an exponentially growing history space up to an indefinite time. On the other hand, a stationary ASM policy $\pi = \phi \circ \nu$, that generates the agent states and actions using time-independent functions, i.e., $S_t \sim \nu(\cdot \mid S_{t-1}, A_{t-1}, O_t)$ and $A_t \sim \phi(\cdot \mid S_t)$ for all $t \geq 1$, can be efficiently implemented for an infinite-horizon environment.

\subsection{ASM policy gradient for infinite-horizon discounted NMDPs}
\label{subsec:ihdpg}
In this section, we first present the policy gradient theorem for infinite-horizon discounted NMDPs that allows us to maximize the discounted reward over the space of given parametric stationary ASM policies. Under a stationary ASM policy $\pi$, we have $(S_t, A_t) \sim \pi(\cdot,\cdot \mid S_{t-1},A_{t-1},O_t), t\ge 1$. We are interested in maximizing the expected cumulative discounted rewards over the parametric class of stationary ASM policies:
\begin{align}\label{obj:ihdpg}
    \max_{\theta \in \bR^d}{J_\gamma(\theta)}, \quad \mbox{where } J_\gamma(\theta) \coloneqq \bE_{\pi_\theta,\mu}\sqbr{R^\gamma_{1:\infty}},
\end{align}
where $\pi_\theta : \cS \times \cA \times \cO \to \Delta_{\cS \times \cA}$ is the stationary ASM policy parameterized by $\theta \in \bR^d$, $d \in \bN$. For a stationary ASM policy, we can also define the corresponding Q functions.
\begin{align*}
    Q^{\gamma,\pi}_t(o_{1:t},a_{0:t},s_t)
    &\coloneqq \bE_\pi\sqbr{R^\gamma_{t:\infty} \mid O_{1:t}=o_{1:t}, A_{0:t}=a_{0:t}, S_t=s_t},
\end{align*}
where $t\ge 1$, and $\bE_\pi$ denotes expectation with respect to the distribution induced by policy $\pi$. Next, we present the policy gradient theorem for the infinite-horizon discounted NMDPs.
\begin{theorem}[ASM Policy Gradient Theorem for Infinite Horizon Discounted NMDPs] \label{thm:ihpg}
    Let $\{\pi_\theta \mid \theta \in \bR^d\}$ be a class of parameterized stationary ASM policies that are differentiable with respect to the parameter vector $\theta$, and the support of $\pi_\theta$ is independent of $\theta$. Suppose that there exists $r_{\max} > 0$ such that $|r_t(\cdot)| \le r_{\max}$ for every $t \ge 1$, and there exists $G > 0$ such that
    \begin{align*}
        \norm{\nabla_\theta \log\br{\pi_\theta(s,a \mid \tilde{s},\tilde{a},o)}}_2 \le G
    \end{align*}
    for every $s,\tilde{s} \in \cS$, $a,\tilde{a} \in \cA$, $o \in \cO$, and $\theta \in \bR^d$. Then we have,
    \begin{align}
        \nabla_\theta J_\gamma(\theta)
        &= \bE_{\pi_\theta,\mu}\sqbr{\sum_{t=1}^{\infty}{\gamma^{t-1} R^\gamma_{t:\infty} \nabla_\theta \log \pi_\theta(S_t,A_t \mid S_{t-1},A_{t-1},O_t)}} 
        \label{eq:ih_pg_return_form}\\
        &= \bE_{\pi_\theta,\mu}\sqbr{\sum_{t=1}^{\infty} \gamma^{t-1} Q^{\gamma,\pi_\theta}_t(O_{1:t},A_{0:t},S_t) \, \nabla_\theta \log \pi_\theta(S_t,A_t \mid S_{t-1},A_{t-1},O_t)}.
        \label{eq:ih_pg_q_form}
    \end{align}
\end{theorem}
See Appendix~\ref{subapp:pg_ihd} for the proof of Theorem~\ref{thm:ihpg}.

Equation~\eqref{eq:ih_pg_return_form} yields an unbiased estimator of the gradient of $J_\gamma(\cdot)$, thereby allowing us to propose the ASMPG algorithm (Algorithm~\ref{algo:ihdnmpg}), which is an online policy gradient algorithm for infinite-horizon discounted reward NMDPs. In practice, the infinite-horizon objective can be approximated by truncating trajectories at a sufficiently large time step $t$, chosen such that $\gamma^{t-1}$ is negligibly small. Under this truncation, the contribution of the discarded tail to the objective becomes negligible.

\begin{algorithm}[t]
    \caption{ASMPG for Infinite-Horizon Discounted NMDPs}
    \label{algo:ihdnmpg}
    \begin{algorithmic}
        \STATE {\bfseries Input} Discount factor $\gamma$, stepsizes $\{\alpha_k\}_{k\geq 1}$, initial parameter $\theta_1$
        \FOR{$k=1,2,\ldots$}
            \STATE Run stationary ASM policy $\pi_{\theta_k}$, and collect trajectory $(o^k_{1:\infty}, a^k_{0:\infty}, s^k_{0:\infty})$
            \STATE Compute the gradient estimate \vspace{-5pt}
            $$\widehat{g}_{\gamma,k} = \sum_{t=1}^{\infty}{\sum_{t\up=t}^{\infty}{\gamma^{t\up-1} r_{t\up}(o^k_{1:t\up},a^k_{1:t\up})} \, \nabla_\theta \log \pi_{\theta_k}(s^k_t,a^k_t \mid s^k_{t-1},a^k_{t-1},o^k_t)}$$ \vspace{-5pt}
            \STATE Update $\theta_{k+1} = \theta_k + \alpha_k \widehat{g}_{\gamma,k}$
        \ENDFOR
    \end{algorithmic}
\end{algorithm}

The second expression of the gradient~\eqref{eq:ih_pg_q_form} leads to an actor-critic algorithm, where the agent maintains a critic to estimate $Q^{\gamma,\pi_\theta}_t$. However, for NMDPs, it is not known how to efficiently estimate $Q^{\gamma,\pi_\theta}_t$, since its domain is the space of history sequences that grows exponentially with time. We defer the design of an efficient actor–critic algorithm for ASM policy learning to future work.

\paragraph{Convergence guarantees.} We obtain analogous convergence guarantees for the infinite-horizon discounted setting, where $1/(1-\gamma)$ serves as the effective horizon. In particular, the time-averaged squared gradient norm converges to zero at a rate of $\tilde{O}((1-\gamma)^{-3.5}/\sqrt{K})$. Detailed statements and proofs are deferred to Appendix~\ref{subapp:proof_conv_discounted}.

\section{Simulation experiments}
\label{sec:experiments}

We evaluate ASMPG on five non-Markovian environments and compare it with the AIS-KL and AIS-MMD methods of \citet{subramanian2022approximate}. The goal is to assess whether directly optimizing the recursive agent-state dynamics and the control policy for return improves empirical performance over information-state baselines.

\subsection{Environments}
We briefly describe the NMDP environments considered in our experiments; detailed descriptions are provided in Appendix~\ref{app:simulation_details}.

\textbf{CheeseMaze.} The CheeseMaze environment is a partially-observable navigation problem with masked states proposed in \citet{mccallum1993overcoming}. The environment consists of $11$ states and $7$ aliased observations. The objective is to reach the goal state. The agent only receives a reward of $1$ when the goal state is reached.

\textbf{HallwayNavigation.} The HallwayNavigation task is again a partially-observable navigation problem introduced by \citet{mccallum1995instance}, where the agent can observe the local wall configuration. The objective is to reach the central goal state, which yields a reward of $5.0$; attempting to move into a wall incurs a reward of $-1.0$, while all other transitions incur a reward of $-0.1$.

\textbf{HealthcareTreatment.} We consider a treatment planning environment~\citep{liu2017deep} in which patients' health evolves according to dynamics that depend on two unobserved history-dependent variables, namely, toxicity burden and treatment-resistance. The agent must choose one of three treatment regimes at each decision epoch to maximize the patient's overall health, along with a binary outcome-oriented large terminal reward or penalty.

\textbf{MachineRepair.} We consider a machine maintenance environment~\citep{eckles1968optimum} in which the agent observes only a coarse machine condition, healthy or degraded, and chooses between continuing operation and performing repair. The deterioration of the machine's health and the effectiveness of the repair depend on a hidden wear level that evolves with the usage and repair history. Continuing the operation in the healthy state yields a positive reward, whereas in the degraded state it yields a negative reward, and each repair incurs a fixed cost.

\textbf{VelocityOnlyCartPole.} We consider a partially observed variant of the classic CartPole control problem \citep{barto1983neuronlike} in which the agent observes only the linear velocity of the cart and the angular velocity of the pole, and its objective is to keep the pole balanced upright for as long as possible. At each step, the agent applies a left or right force to the cart. The agent receives a unit reward at each time step for keeping the pole upright and the system within its stability limits.

\subsection{Experiment setup}
ASMPG uses two neural networks: one for the agent-state dynamics and one for the control policy. The state-update network takes $(S_{t-1},A_{t-1},O_t)$ as input and outputs a $|\cS|$-dimensional vector of logits. A softmax over these logits defines the distribution over the discrete agent state $S_t$. This network is a multilayer perceptron (MLP) with three hidden layers of widths $2d_h$, $2d_h$, and $d_h$, where $d_h$ is the base hidden dimension. The policy network maps $S_t$ to action logits using an MLP with two hidden layers of width $d_h$, followed by a softmax over $\cA$. The values of $|\cS|$ and $d_h$ used in each environment are reported in Table~\ref{tab:params} of Appendix~\ref{app:simulation_details}.

We train both networks using the ASMPG estimator for the  discounted return corresponding to Eq.~\eqref{obj:ihdpg}. All methods use discount factor $\gamma=0.99$, and episodes are truncated after $200$ steps if they have not been terminated earlier. The ASMPG parameters are optimized with Adam using a learning rate of $10^ {-3}$. Each update averages gradient estimates over $10$ episodes. We also add a log-barrier regularizer with coefficient $0.01$ to encourage exploration by keeping the learned probabilities away from the boundary.

For AIS-KL and AIS-MMD, we use the implementation and standardized hyperparameters provided by \citet{subramanian2022approximate}. To make the representation sizes comparable, we set the AIS information-state dimension equal to $|\cS|$ in each environment.

\subsection{Performance comparison}
\begin{figure}[t]
    \centering
    \begin{subfigure}[b]{0.31\textwidth}
        \centering
        \includegraphics[width=\textwidth]{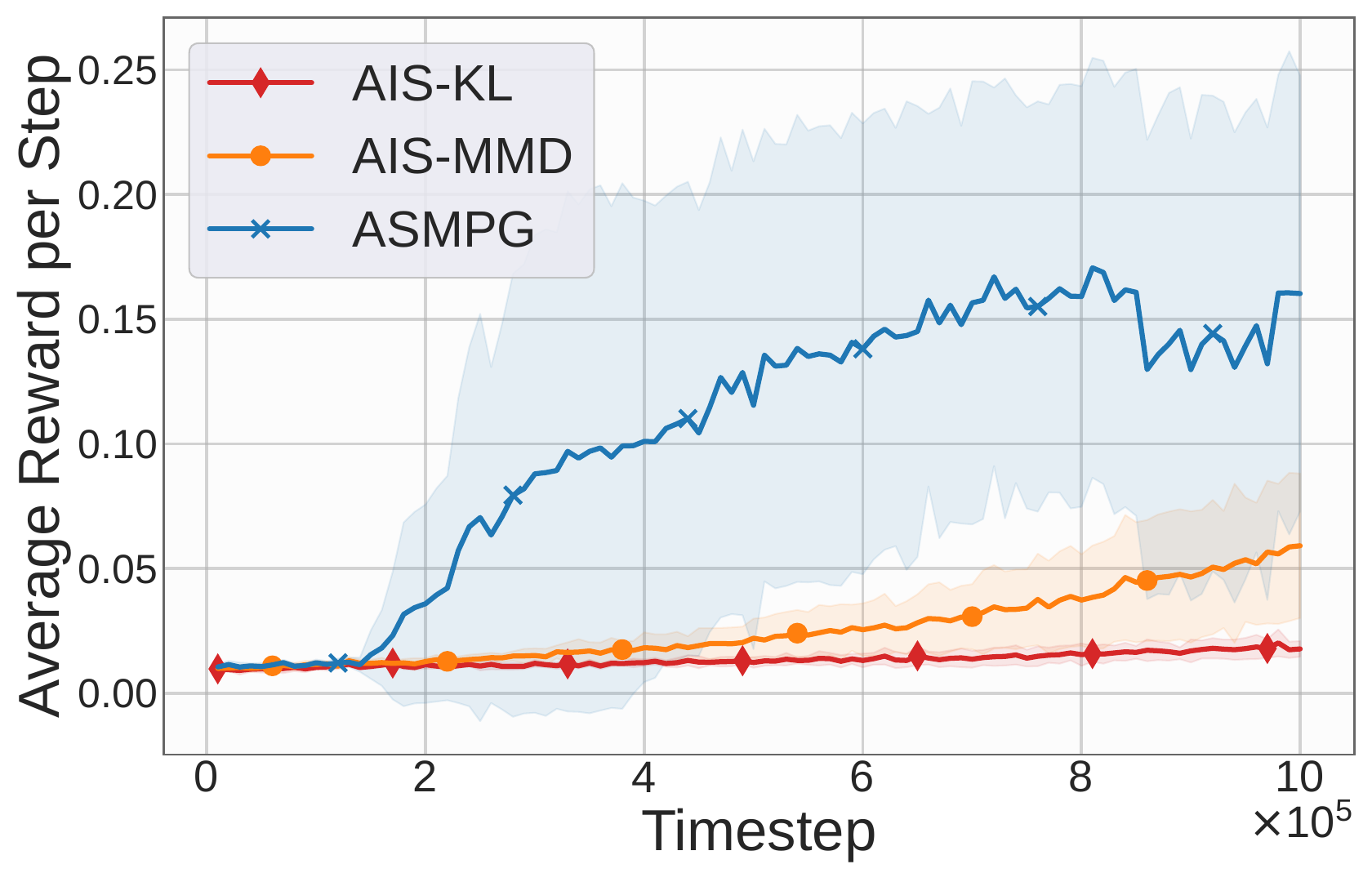}
        \caption{CheeseMaze}
        \label{fig:cheesemaze_p1}
    \end{subfigure}
    \hfill
    \begin{subfigure}[b]{0.31\textwidth}
        \centering
        \includegraphics[width=\textwidth]{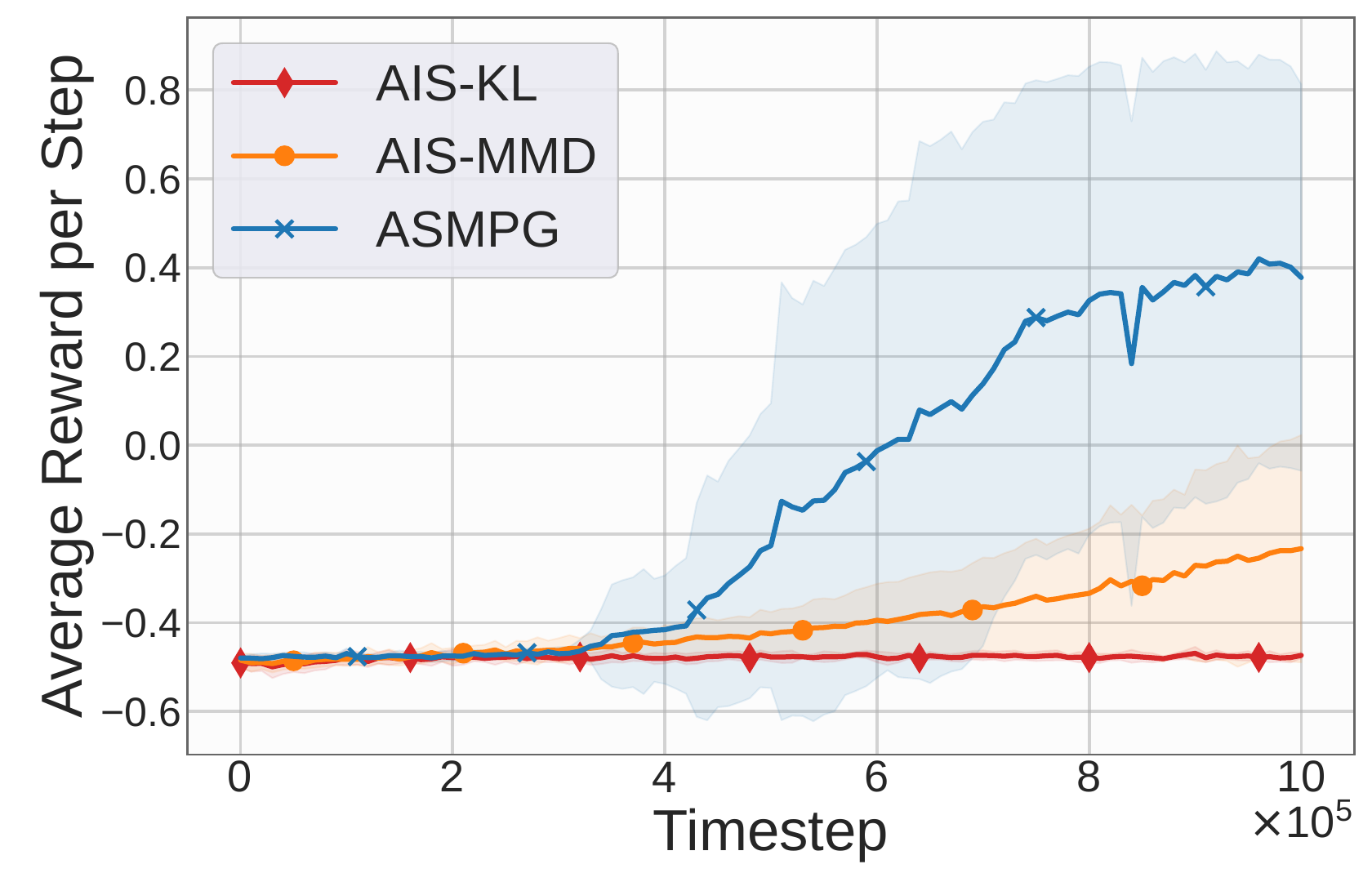}
        \caption{HallwayNavigation}
        \label{fig:hallway_navigation_p1}
    \end{subfigure}
    \hfill
    \begin{subfigure}[b]{0.31\textwidth}
        \centering
        \includegraphics[width=\textwidth]{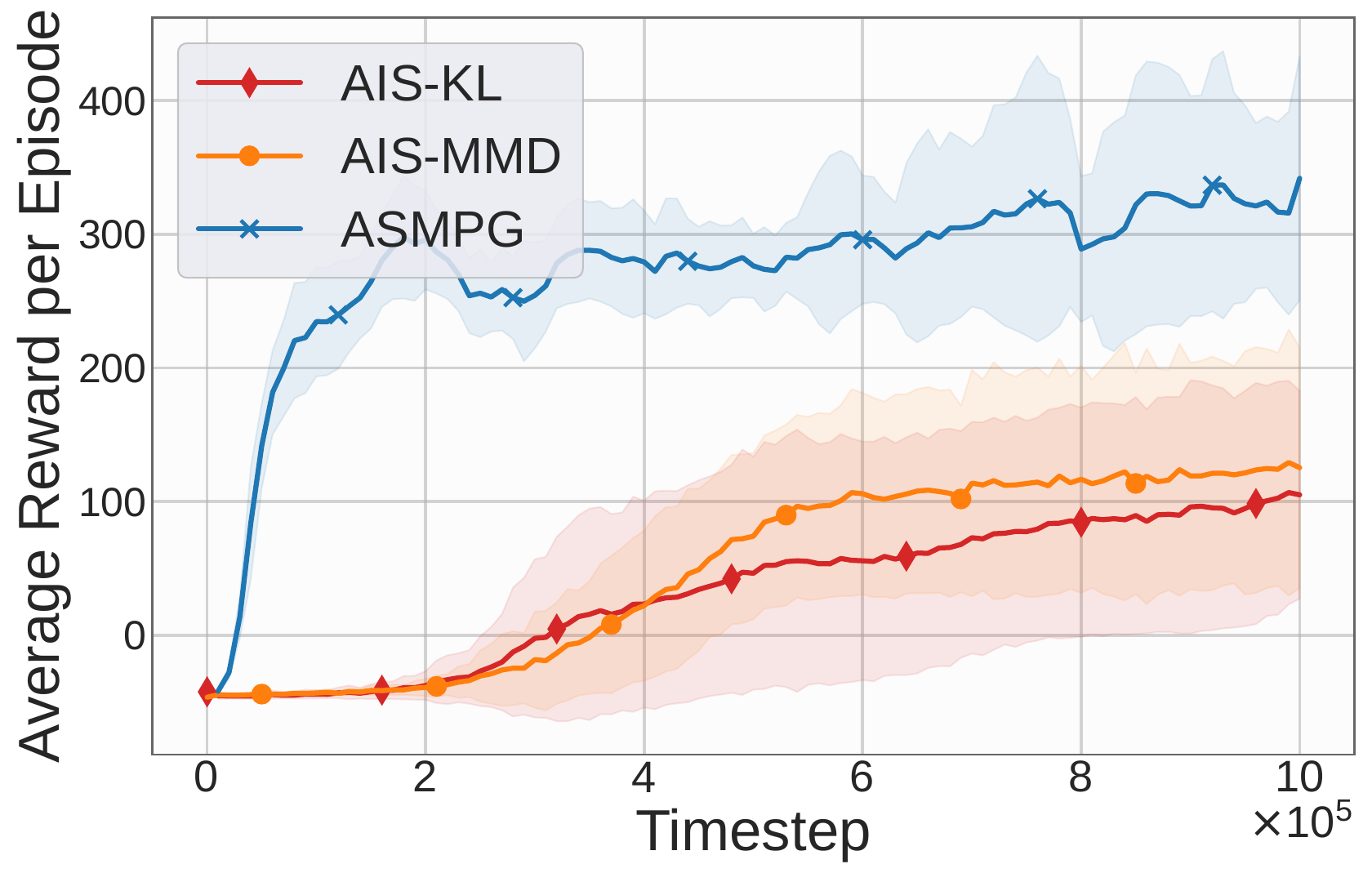}
        \caption{HealthcareTreatment}
        \label{fig:treatment_planning_p1}
    \end{subfigure}

    \vspace{0.5em}

    \begin{subfigure}[b]{0.31\textwidth}
        \centering
        \includegraphics[width=\textwidth]{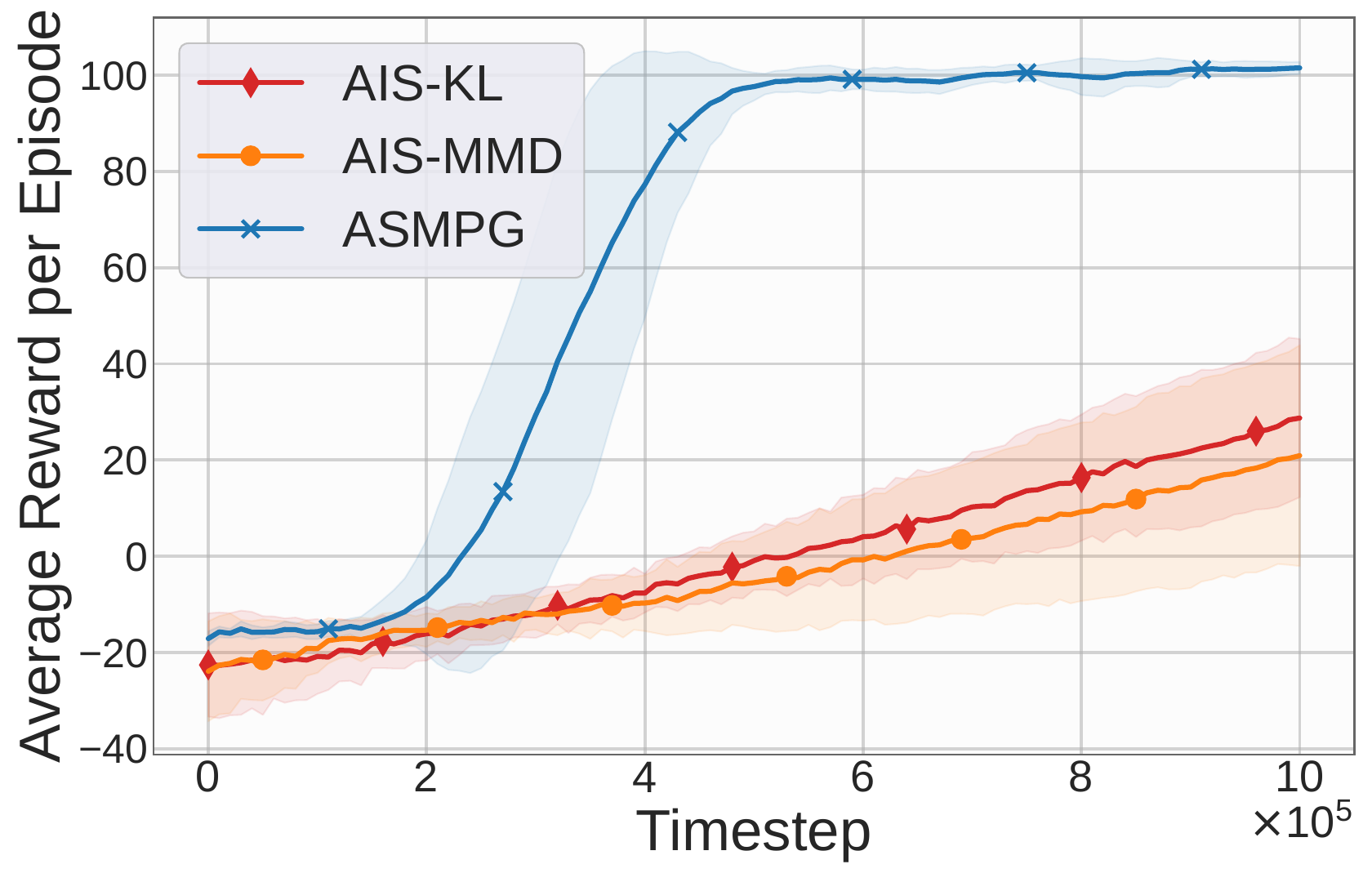}
        \caption{MachineRepair}
        \label{fig:machine_repair_p1}
    \end{subfigure}
    \hspace{30pt}
    \begin{subfigure}[b]{0.31\textwidth}
        \centering
        \includegraphics[width=\textwidth]{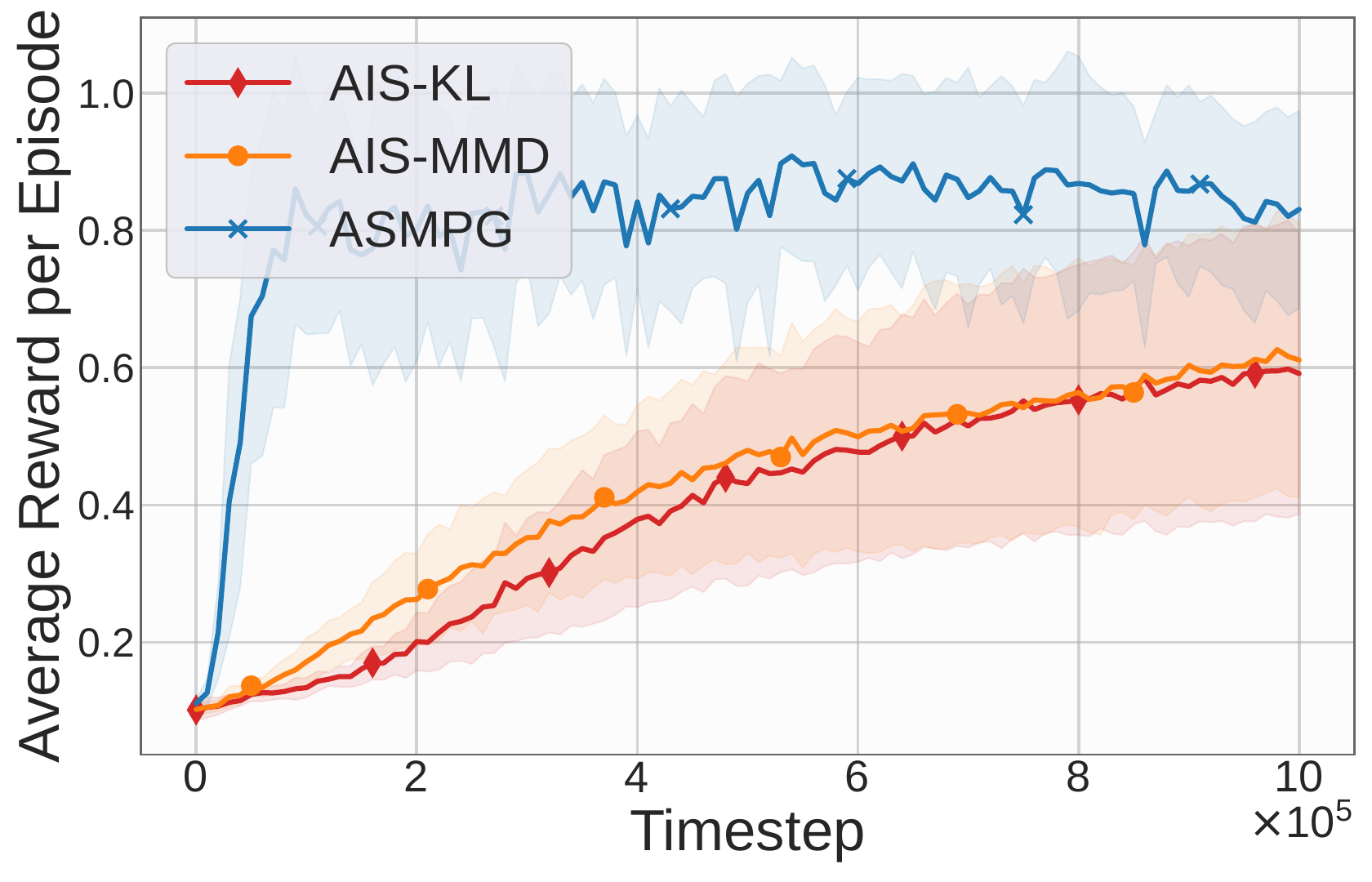}
        \caption{VelocityOnlyCartPole}
        \label{fig:velocity_cartpole_p1}
    \end{subfigure}
    \caption{Learning curves for ASMPG, AIS-KL, and AIS-MMD on the five environments. Curves show averages over $10$ random seeds; shaded bands indicate variability across seeds.}
    \label{fig:learning_curves}
\end{figure}
\begin{figure}[t]
    \centering
    \begin{subfigure}[b]{0.18\textwidth}
        \centering
        \includegraphics[width=\textwidth]{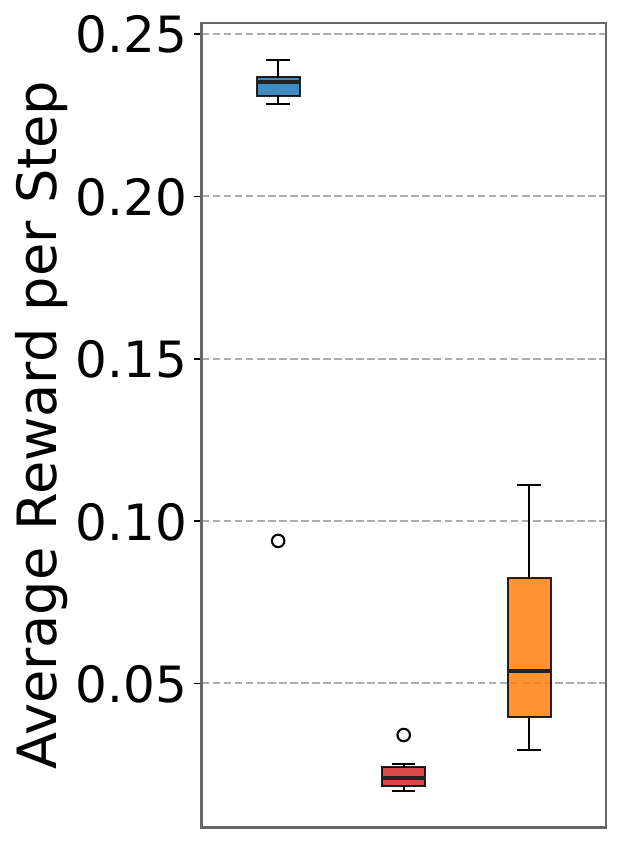}
        \caption{\centering CheeseMaze\\~}
        \label{fig:cheesemaze_p2}
    \end{subfigure}
    \hfill
    \begin{subfigure}[b]{0.19\textwidth}
        \centering
        \includegraphics[width=\textwidth]{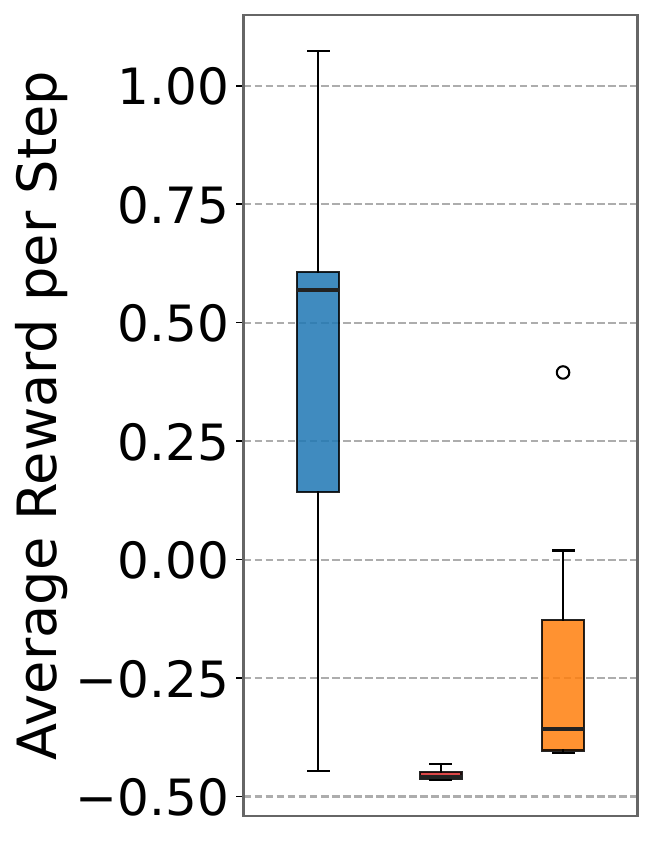}
        \caption{\centering HallwayNavigation}
        \label{fig:hallway_navigation_p2}
    \end{subfigure}
    \hfill
    \begin{subfigure}[b]{0.18\textwidth}
        \centering
        \includegraphics[width=\textwidth]{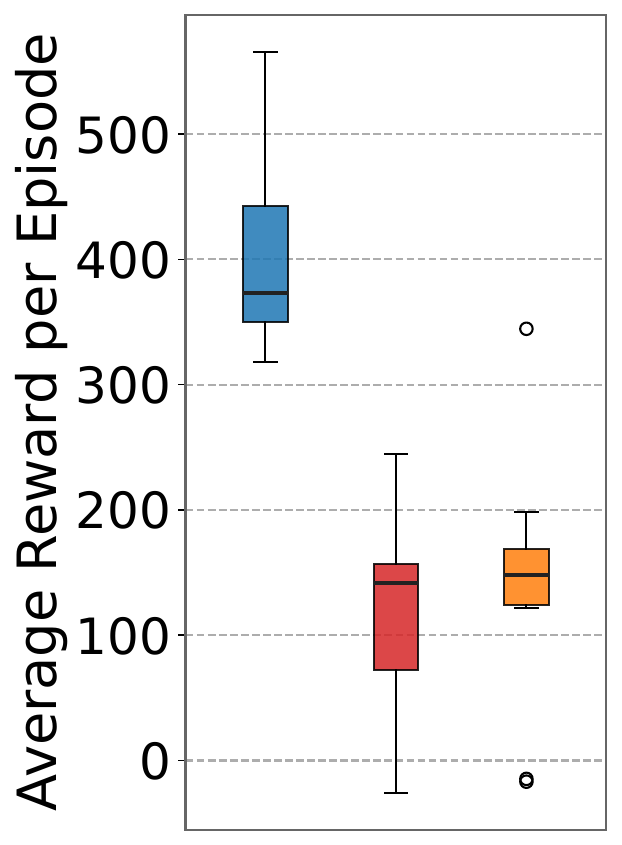}
        \caption{\centering HealthcareTreatment}
        \label{fig:treatment_planning_p2}
    \end{subfigure}
    \hfill
    \begin{subfigure}[b]{0.18\textwidth}
        \centering
        \includegraphics[width=\textwidth]{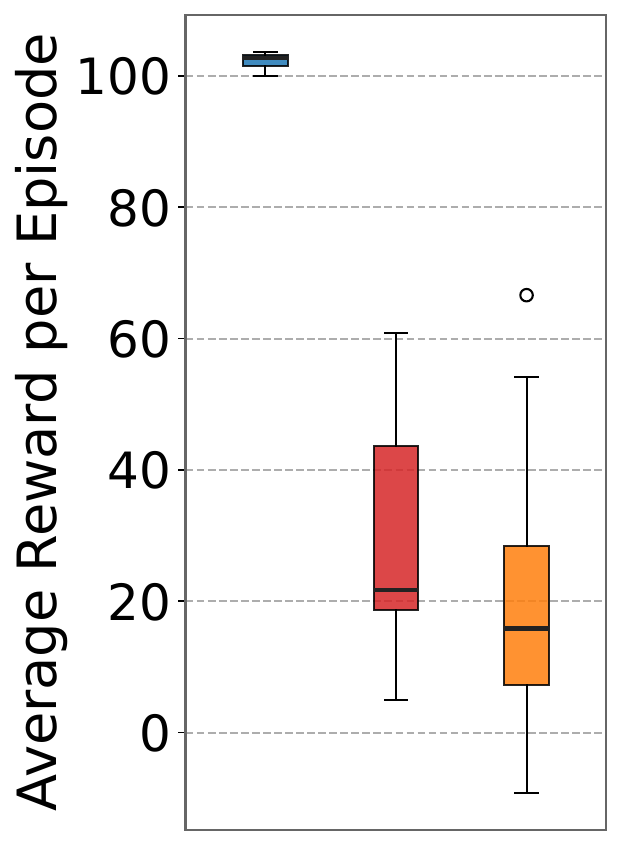}
        \caption{\centering MachineRepair\\~}
        \label{fig:machine_repair_p2}
    \end{subfigure}
    \hfill
    \begin{subfigure}[b]{0.18\textwidth}
        \centering
        \includegraphics[width=\textwidth]{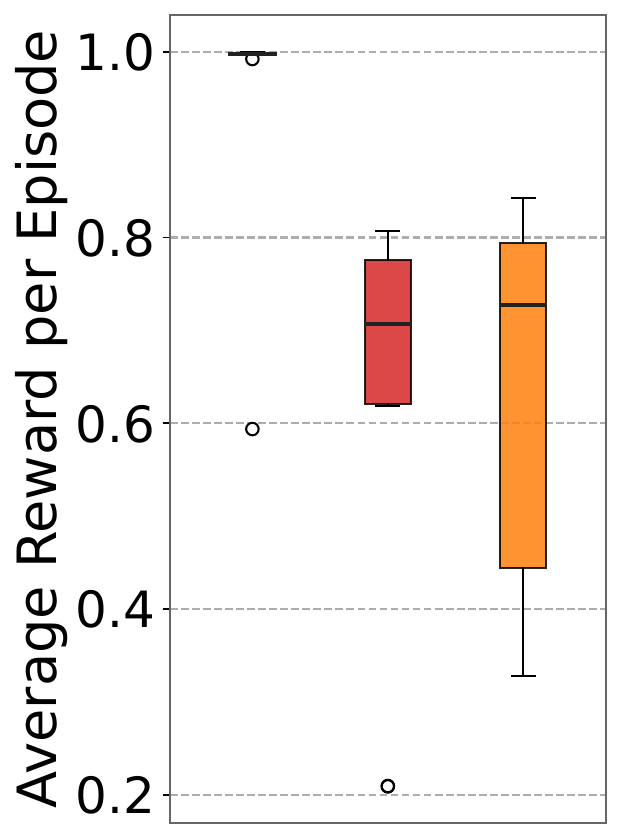}
        \caption{\centering VelocityOnlyCartPole}
        \label{fig:velocity_cartpole_p2}
    \end{subfigure}
    \caption{Best-checkpoint performance over $10$ random seeds. The following colors indicate the associated algorithms: \raisebox{0.25ex}{\colorbox[HTML]{1f77b4}{\color[HTML]{1f77b4}\rule{0.4cm}{0.08cm}}} -- ASMPG, \raisebox{0.25ex}{\colorbox[HTML]{d62728}{\color[HTML]{d62728}\rule{0.4cm}{0.08cm}}} -- AIS-KL, \raisebox{0.25ex}{\colorbox[HTML]{ff7f0e}{\color[HTML]{ff7f0e}\rule{0.4cm}{0.08cm}}} -- AIS-MMD. The box denotes the interquartile range, and the whiskers indicate the range of non-outlier values. Points beyond the whiskers are considered outliers.}
    \label{fig:best_model}
\end{figure}

We train each method for $10^6$ environment steps and evaluate performance every $10{,}000$ steps. The results are averaged over 10 random seeds. For CheeseMaze and HallwayNavigation, the reward per episode is the same for all successful trajectories, regardless of how long the agent takes to reach the goal; thus, the primary challenge is to reach the destination as quickly as possible. Hence, to better differentiate the algorithms' performance, we report average reward per step. For the other environments, we report average reward per episode.

Figure~\ref{fig:learning_curves} shows that ASMPG achieves better performance than AIS-KL and AIS-MMD across all five environments. The larger fluctuations in some ASMPG curves, especially in the navigation tasks, are consistent with the high variance of REINFORCE-type estimators.

Figure~\ref{fig:best_model} compares the best evaluation checkpoint obtained during training for each seed. ASMPG has the highest median performance in all five environments and shows especially strong best-checkpoint performance in CheeseMaze, MachineRepair, and VelocityOnlyCartPole. Overall, Figures~\ref{fig:learning_curves} and~\ref{fig:best_model} support the benefit of directly optimizing the agent-state dynamics and the control policy for the reward objective, rather than learning the representation through an auxiliary predictive criterion.

\section{Conclusion}
\label{sec:conclusion}

We studied policy-gradient methods for reinforcement learning in NMDPs, where observations and rewards may depend on the full interaction history. To handle this history dependence, we introduced Agent State-Markov policies, in which the agent maintains a recursively updated internal state and acts in a Markovian manner with respect to this state. Unlike approaches that fix the agent-state dynamics or learn representations through auxiliary predictive objectives, our formulation directly optimizes both the agent-state dynamics and the control policy for the expected-return objective.

We derived policy-gradient identities for episodic and infinite-horizon discounted NMDPs and used them to develop the ASMPG algorithm. We also established finite-time convergence bounds and almost sure convergence to the set of stationary points under standard regularity assumptions. The simulation results show that ASMPG performs favorably compared with baselines across several non-Markovian environments, supporting the benefit of reward-centric joint optimization of agent state dynamics and control policy.

The proposed method is based on REINFORCE-type gradient estimators and may therefore suffer from high variance. A natural direction for future work is to develop actor-critic variants for ASM policies that reduce variance while preserving the same reward-centric objective. Another important direction is to study conditions under which policy-gradient methods for non-Markovian reinforcement learning converge to globally optimal policies, extending recent global convergence results from the Markovian setting.
\appendix
\section{Extended related work}
\label{app:extended_related_work}
Non-Markovian reinforcement learning has been studied through several complementary notions of state, memory, and temporal structure. The classical example is the POMDP, where the environment has a Markov latent state, but the agent observes only a noisy signal. If the model is known, the belief over latent states is a sufficient statistic, reducing the problem to a fully observed belief-state MDP~\citep{astrom1965optimal, smallwood1973optimal, kaelbling1998planning, poupart2011closing}. This reduction is conceptually powerful, but belief construction requires a known or accurately estimated transition-observation model and may be computationally expensive.

When the latent model is unavailable, one can replace the full history with a learned or hand-designed memory. Early approaches include finite observation windows~\citep{loch1998using} and finite-state controllers~\citep{poupart2003bounded}. Predictive state representations avoid latent variables by representing state through predictions of future tests~\citep{littman2001predictive, singh2003learning}. In deep RL, recurrent architectures provide a flexible way to encode histories into latent states~\citep{rumelhart1986learning, hochreiter1997long, schmidhuber1990reinforcement, bakker2001reinforcement}. Recent empirical work shows that recurrent model-free RL can be a strong baseline for many POMDPs when architectures and hyperparameters are carefully tuned~\citep{ni2022recurrent}, and POPGym provides a benchmark suite for comparing memory mechanisms in partially observable RL~\citep{morad2023popgym}. These approaches are expressive, but their learned memory is often difficult to analyze.

A more structured line of work studies agent-side state representations. \citet{dong2022simple} introduce recursively updated agent states and analyze efficient learning when the state dynamics are fixed. \citet{sinha2024periodic} also use fixed agent states and learn periodic policies for POMDPs. Approximate information states (AIS)~\citep{subramanian2022approximate} provide a general framework in which a history representation is approximately sufficient for rewards and future prediction, yielding approximate planning and learning guarantees. \citet{chandak2024reinforcement} use related information-state ideas to quantify the error caused by non-Markovianity in Q-learning. Our approach is closest in spirit to these agent-state and information-state methods, but differs in the optimization objective: rather than fixing the state dynamics or learning them through an auxiliary sufficiency/prediction criterion, we optimize the recursive state dynamics directly for cumulative reward.

Another complementary direction uses automata or formal languages to encode temporal structure. Regular decision processes and their offline/online variants model history-dependent behavior using regular-language structure~\citep{brafman2019regular,abadi2020learning,ronca2021efficient,cipollone2023provably,deb2024tractable}. Reward machines encode non-Markovian reward structure as finite-state machines and can expose decompositions useful for learning~\citep{icarte2018using}. These approaches are most natural when the relevant temporal abstraction is specified, symbolic, or inferred as an automaton. ASM policies instead treat the agent state as a stochastic, differentiable memory mechanism that is optimized jointly with the policy.

Our analysis is also closely related to policy-gradient methods under partial observability. GPOMDP provides simulation-based gradient estimates for average-reward POMDPs without access to the latent state~\citep{baxter2001infinite}. Gradient methods for finite-state controllers learn both action choices and stochastic memory updates~\citep{meuleau1999learning}, while actor-critic methods for structured POMDP policies estimate policy gradients using value-function approximation~\citep{yu2005function}. \citet{aberdeen2007policy} combine policy gradients with sufficient-statistic tracking, including beliefs and predictive state representations. Recurrent policy gradients optimize RNN or LSTM policies directly through memory~\citep{wierstra2010recurrent}, and recent work gives finite-time analyses for natural actor-critic methods with finite memory or recurrent policies in POMDPs~\citep{cayci2024finite,cayci2025recurrent}. These works show that policy-gradient methods can be applied under partial observability, but they are formulated for POMDPs, sufficient statistics, finite-state controllers, or generic recurrent policies. The present work targets general NMDPs and derives a policy-gradient theorem for the ASM class, enabling joint reward-centric optimization of the agent-state dynamics and the control policy.

Finally, our convergence analysis builds on the mature theory of policy-gradient methods for fully observed MDPs, including classical likelihood-ratio estimators~\citep{williams1992simple, sutton1999policy} and recent non-asymptotic/global convergence analyses~\citep{agarwal2021theory, zhang2021sample, mei2020global, mei2023stochastic}. The main technical difference is that in NMDPs, rewards and observations depend on the entire history; the ASM recursion provides enough structure to obtain a usable gradient expression and finite-time stationarity guarantees without reducing the environment to a Markov state observed by the learner. 
\section{Results related to episodic NMDPs}
\label{app:episodic_nmdp}
\subsection{Policy gradient theorem}
\label{subapp:pg_episodic}
For a fixed episodic ASM policy $\pi$, define the Q-function and value function as follows.
\begin{align*}
    Q^{\pi}_t(o_{1:t},a_{0:t},s_t)
    &\coloneqq \bE_\pi\sqbr{R_{t:H} \mid O_{1:t}=o_{1:t}, A_{0:t}=a_{0:t}, S_t=s_t}, \\
    V^{\pi}_t(o_{1:t},a_{0:t-1},s_{t-1})
    &\coloneqq \bE_\pi\sqbr{R_{t:H} \mid O_{1:t}=o_{1:t}, A_{0:t-1}=a_{0:t-1}, S_{t-1}=s_{t-1}},
\end{align*}
for $t=1,2,\ldots, H$, where $R_{t:H} = \sum_{t\up=t}^{H}{r_{t\up}(O_{1:t\up},A_{1:t\up})}$ and $\bE_\pi$ denotes the expectation taken considering that the agent states and actions from time $t$ onward are generated by $\pi$.

\begin{proposition}
    \label{prop:fhpg}
    Let $\{\pi_\theta = \{\pi_{\theta,t}\}_{t=1}^H \mid \theta \in \bR^d\}$ be a class of parameterized episodic ASM policies that are differentiable with respect to the parameter vector $\theta$ and the support of $\pi_{\theta,t}$ is independent of $\theta$ for all $t \in [H]$ and $\theta \in \bR^d$. Then, for every $t \in [H]$, $o_{1:t} \in \cO^t, a_{0:t-1} \in \cA^t$, and $s_{t-1} \in \cS$,
    \begin{align}
        &\nabla_\theta{V^{\pi_\theta}_t(o_{1:t}, a_{0:t-1},s_{t-1})} \notag\\
        &= \bE_{\pi_\theta}\sqbr{\sum_{t\up=t}^{H}{R_{t\up:H} \nabla_\theta \log\br{\pi_{\theta,t\up}(S_{t\up}, A_{t\up} \mid S_{t\up-1}, A_{t\up-1}, O_{t\up})}} \mid o_{1:t}, a_{0:t-1}, s_{t-1}} \label{pg:1}\\
        &= \bE_{\pi_\theta}\sqbr{\sum_{t\up=t}^{H}{Q^{\pi_\theta}_{t\up}(O_{1:t\up}, A_{0:t\up}, S_{t\up}) \nabla_\theta \log\br{\pi_{\theta,t\up}(S_{t\up}, A_{t\up} \mid S_{t\up-1}, A_{t\up-1}, O_{t\up})}} \mid o_{1:t}, a_{0:t-1}, s_{t-1}}. \label{pg:2}
    \end{align}
\end{proposition}

\begin{proof}
    Fix $t \in [H]$ and let
    \begin{align*}
        h_t \coloneqq (o_{1:t}, a_{0:t-1}, s_{t-1}).
    \end{align*}
    Note that
    \begin{align*}
        V^{\pi_\theta}_t(h_t) = \sum_{\substack{\tau = (o_{t+1:H}, a_{t:H}, s_{t:H}) \\ \in \cO^{H-t} \times \cA^{H-t+1} \times \cS^{H-t+1}}}{\sum_{t\up=t}^{H}{r_{t\up}(o_{1:t\up},a_{1:t\up})} \bP^{\pi_\theta}(\tau \mid h_t)},
    \end{align*}
    where $\bP^{\pi_\theta}(\cdot \mid h_t)$ denotes the probability measure on the trajectory space induced by the application of policy $\pi_\theta$ and conditioned on $(O_{1:t}, A_{0:t-1}, S_{t-1}) = h_t$. We rewrite the gradient of the value function as follows.
    \begin{align}
        \nabla_\theta V^{\pi_\theta}_t(h_t) &\stackrel{(a)}{=} \sum_{\tau}{\sum_{t\up=t}^{H} r_{t\up}(o_{1:t\up},a_{1:t\up}) \nabla_\theta\bP^{\pi_\theta}(\tau \mid h_t)} \notag\\
        &\stackrel{(b)}{=} \sum_{\substack{\tau:\\ \bP^{\pi_\theta}(\tau \mid h_t)>0}}{\sum_{t\up=t}^{H} r_{t\up}(o_{1:t\up},a_{1:t\up}) \nabla_\theta\log\br{\bP^{\pi_\theta}(\tau \mid h_t)} \bP^{\pi_\theta}(\tau \mid h_t)}, \label{eq:fh_log_derivative}
    \end{align}
    where $(a)$ follows from the fact that $\sum_{t\up=t}^{H}{r_{t\up}(o_{1:t\up},a_{1:t\up})}$ does not depend on $\theta$, and $(b)$ uses the log-derivative trick~\citep{williams1992simple}. The equality in $(b)$ is valid because the support of the policy, and hence the support of the finite trajectory law, does not depend on $\theta$.

    The probability of trajectory $\tau$ conditioned on $h_t$, under the probability measure induced by the application of policy $\pi_\theta$, can be written as
    \begin{align*}
        \bP^{\pi_\theta}(\tau \mid h_t)
        =
        \br{\prod_{t\up=t}^{H}{\pi_{\theta,t\up}(s_{t\up}, a_{t\up} \mid s_{t\up-1}, a_{t\up-1}, o_{t\up})}}
        \br{\prod_{t\up=t}^{H-1}{p_{t\up}(o_{t\up+1} \mid o_{1:t\up}, a_{1:t\up})}}.
    \end{align*}
    On the support of this conditional trajectory law, its logarithm simplifies into a sum of two terms as follows:
    \begin{align*}
        \log\br{\bP^{\pi_\theta}(\tau \mid h_t)} = \sum_{t\up=t}^{H}{\log\br{\pi_{\theta,t\up}(s_{t\up}, a_{t\up} \mid s_{t\up-1}, a_{t\up-1}, o_{t\up})}} + \sum_{t\up=t}^{H-1}{\log\br{p_{t\up}(o_{t\up+1} \mid o_{1:t\up}, a_{1:t\up})}},
    \end{align*}
    where only the first term depends on $\theta$. Thus,
    \begin{align*}
        \nabla_\theta \log\br{\bP^{\pi_\theta}(\tau \mid h_t)}
        =
        \sum_{t\up=t}^{H}{\nabla_\theta \log\br{\pi_{\theta,t\up}(s_{t\up}, a_{t\up} \mid s_{t\up-1}, a_{t\up-1}, o_{t\up})}}.
    \end{align*}
    Invoking the above expression in \eqref{eq:fh_log_derivative}, we obtain that
    \begin{align*}
        \nabla_\theta V^{\pi_\theta}_t(h_t)
        =
        \bE_{\pi_\theta}\sqbr{\sum_{t\up=t}^{H}{R_{t:H} \nabla_\theta \log\br{\pi_{\theta,t\up}(S_{t\up}, A_{t\up} \mid S_{t\up-1}, A_{t\up-1}, O_{t\up})}} \mid h_t}.
    \end{align*}
    Note that $R_{t:H} = R_{t:t\up-1} + R_{t\up:H}$. Hence, the proof of \eqref{pg:1} will be complete if we show that
    \begin{align*}
        \bE_{\pi_\theta}\sqbr{\sum_{t\up=t}^{H}{R_{t:t\up-1} \nabla_\theta \log\br{\pi_{\theta,t\up}(S_{t\up}, A_{t\up} \mid S_{t\up-1}, A_{t\up-1}, O_{t\up})}} \mid h_t} = 0.
    \end{align*}
    Denote the $\sigma$-field
    \begin{align*}
        \cF_{t\up}
        \coloneqq
        \sigma(O_1, S_1, A_1, O_2, S_2, A_2, \ldots, O_{t\up-1}, S_{t\up-1}, A_{t\up-1}, O_{t\up}).
    \end{align*}
    Note that $R_{t:t\up-1}$ is $\cF_{t\up}$-measurable. For every $t\up \in \{t,\ldots,H\}$, using the tower property of conditional expectation~\citep{billingsley2017probability}, we have
    \begin{align*}
        &\bE_{\pi_\theta}\sqbr{R_{t:t\up-1} \nabla_\theta \log\br{\pi_{\theta,t\up}(S_{t\up}, A_{t\up} \mid S_{t\up-1}, A_{t\up-1}, O_{t\up})} \mid h_t} \\
        &= \bE_{\pi_\theta}\sqbr{R_{t:t\up-1} \bE_{\pi_\theta}\sqbr{\nabla_\theta \log\br{\pi_{\theta,t\up}(S_{t\up}, A_{t\up} \mid S_{t\up-1}, A_{t\up-1}, O_{t\up})} \mid \cF_{t\up}} \mid h_t} \\
        &= \bE_{\pi_\theta}\Bigg[R_{t:t\up-1} \nabla_\theta \sum_{\substack{(s,a) \in \cS \times \cA:\\ \pi_{\theta,t\up}(s,a \mid S_{t\up-1}, A_{t\up-1}, O_{t\up})>0}}{\pi_{\theta,t\up}(s,a \mid S_{t\up-1}, A_{t\up-1}, O_{t\up})} \mid h_t\Bigg] \\
        &= 0.
    \end{align*}
    The last equality follows because the support of $\pi_{\theta,t\up}(\cdot,\cdot \mid S_{t\up-1}, A_{t\up-1}, O_{t\up})$ is independent of $\theta$, and the sum of the policy probabilities over this support is one. Summing the above identity over $t\up=t,\ldots,H$ gives
    \begin{align*}
        \bE_{\pi_\theta}\sqbr{\sum_{t\up=t}^{H}{R_{t:t\up-1} \nabla_\theta \log\br{\pi_{\theta,t\up}(S_{t\up}, A_{t\up} \mid S_{t\up-1}, A_{t\up-1}, O_{t\up})}} \mid h_t} = 0.
    \end{align*}
    Thus, we have \eqref{pg:1} proved. Now, we will prove~\eqref{pg:2}. By definition,
    \begin{align*}
        Q^{\pi_\theta}_{t\up}(O_{1:t\up}, A_{0:t\up}, S_{t\up})
        =
        \bE_{\pi_\theta}\sqbr{R_{t\up:H} \mid O_{1:t\up}, A_{0:t\up}, S_{t\up}}.
    \end{align*}
    By the recursive structure of ASM policies, conditioning on the earlier agent-state history beyond $(O_{1:t\up}, A_{0:t\up}, S_{t\up})$ does not change the conditional distribution of the future trajectory. Hence,
    \begin{align*}
        \bE_{\pi_\theta}\sqbr{R_{t\up:H} \mid \cF_{t\up}, S_{t\up}, A_{t\up}} = Q^{\pi_\theta}_{t\up}(O_{1:t\up}, A_{0:t\up}, S_{t\up}).
    \end{align*}
    Also, note that $\nabla_\theta \log\br{\pi_{\theta,t\up}(S_{t\up}, A_{t\up} \mid S_{t\up-1}, A_{t\up-1}, O_{t\up})}$ is $\sigma(\cF_{t\up}, S_{t\up}, A_{t\up})$-measurable.~Now, again using the tower property of conditional expectation in \eqref{pg:1}, we obtain that
    \begin{align*}
        &\bE_{\pi_\theta}\sqbr{\sum_{t\up=t}^{H}{R_{t\up:H} \nabla_\theta \log\br{\pi_{\theta,t\up}(S_{t\up}, A_{t\up} \mid S_{t\up-1}, A_{t\up-1}, O_{t\up})}} \mid h_t} \\
        &= \bE_{\pi_\theta}\sqbr{\sum_{t\up=t}^{H}{\bE_{\pi_\theta}\sqbr{R_{t\up:H} \nabla_\theta \log\br{\pi_{\theta,t\up}(S_{t\up}, A_{t\up} \mid S_{t\up-1}, A_{t\up-1}, O_{t\up})} \mid \cF_{t\up}, S_{t\up}, A_{t\up}}} \mid h_t} \\
        &= \bE_{\pi_\theta}\sqbr{\sum_{t\up=t}^{H}{\bE_{\pi_\theta}\sqbr{R_{t\up:H} \mid \cF_{t\up}, S_{t\up}, A_{t\up}} \nabla_\theta \log\br{\pi_{\theta,t\up}(S_{t\up}, A_{t\up} \mid S_{t\up-1}, A_{t\up-1}, O_{t\up})}} \mid h_t} \\
        &= \bE_{\pi_\theta}\sqbr{\sum_{t\up=t}^{H}{Q^{\pi_\theta}_{t\up}(O_{1:t\up}, A_{0:t\up}, S_{t\up}) \nabla_\theta \log\br{\pi_{\theta,t\up}(S_{t\up}, A_{t\up} \mid S_{t\up-1}, A_{t\up-1}, O_{t\up})}} \mid h_t}.
    \end{align*}
    This completes the proof.
\end{proof}

\begin{theorem}[Theorem~\ref{thm:fhpg} restated]
    \label{thm:pg_restate}
    Let $\{\pi_\theta = \{\pi_{\theta,t}\}_{t=1}^H \mid \theta \in \bR^d\}$ be a class of parameterized episodic ASM policies that are differentiable with respect to the parameter vector $\theta$ and the support of $\pi_{\theta,t}$ is independent of $\theta$ for all $t \in [H]$ and $\theta \in \bR^d$. Then
    \begin{align*}
        \nabla_\theta J(\theta) &= \bE_{\pi_\theta}\sqbr{\sum_{t=1}^{H}{R_{t:H} \nabla_\theta \log \pi_{\theta,t}(S_t,A_t \mid S_{t-1},A_{t-1},O_t)}} \\
        &= \bE_{\pi_\theta}\sqbr{\sum_{t=1}^{H}{Q^{\pi_\theta}_t(O_{1:t},A_{0:t},S_t) \, \nabla_\theta \log \pi_{\theta,t}(S_t,A_t \mid S_{t-1},A_{t-1},O_t)}}.
    \end{align*}
\end{theorem}

\begin{proof}
    Observe that, 
    \begin{align*}
        J(\theta) &= \bE\sqbr{V^{\pi_\theta}_1(O_1,A_0,S_0)} = \sum_{o \in \cO}{\mu(o) V^{\pi_\theta}_1(o,a_0,s_0)},
    \end{align*}
    where $\mu(\cdot)$ is the distribution of the initial observation $O_1$. Hence,
    \begin{align*}
        \nabla_\theta{J(\theta)} = \sum_{o \in \cO}{\mu(o) \nabla_{\theta}V^{\pi_\theta}_1(o,a_0,s_0)}.
    \end{align*}
    The proof is complete by using~\eqref{pg:1} and~\eqref{pg:2} from Proposition~\ref{prop:fhpg} with $t$ set equal to $1$.
\end{proof}

\subsection{Smoothness of policy gradient objective}
\label{subapp:smooth_episodic}

\begin{proof}[\textbf{Proof of Lemma~\ref{lem:finite_smooth}}]
For notational compactness, define, for $t\in[H]$,
\[
    Z_t(\theta)
    :=
    \nabla_\theta
    \log
    \pi_{\theta,t}
    \left(
        S_t,A_t
        \mid
        S_{t-1},A_{t-1},O_t
    \right),
\]
and
\[
    B_t(\theta)
    :=
    \nabla_\theta^2
    \log
    \pi_{\theta,t}
    \left(
        S_t,A_t
        \mid
        S_{t-1},A_{t-1},O_t
    \right).
\]
By Assumption~\ref{assum:smooth_parame_fh},
\[
    \|Z_t(\theta)\|_2\le G,
    \qquad
    \|B_t(\theta)\|_2\le M
\]
for every $t\in[H]$ and every value of the variables on the support of the policy.

By the $Q$-function form of the episodic ASM policy-gradient theorem,
\[
    \nabla_\theta J(\theta)
    =
    \bE_{\pi_\theta}
    \left[
        \sum_{t=1}^H
        Q_t^{\pi_\theta}(O_{1:t},A_{0:t},S_t)
        Z_t(\theta)
    \right].
\]
We now differentiate this identity. To make the differentiation precise, write the expectation over
the full prefix
\[
    \mathcal{H}_t := (O_{1:t},S_{0:t},A_{0:t}).
\]
The probability of such a prefix under $\pi_\theta$ is a product of the environment kernels, which
do not depend on $\theta$, and the ASM policy kernels
\[
    \prod_{j=1}^t
    \pi_{\theta,j}
    \left(
        S_j,A_j
        \mid
        S_{j-1},A_{j-1},O_j
    \right).
\]
Hence, on the support of the prefix law,
\[
    \nabla_\theta
    \log
    P_{\pi_\theta}(\mathcal{H}_t)
    =
    \sum_{j=1}^t Z_j(\theta).
\]
Using the regularity assumptions from Theorem~\ref{thm:fhpg}, we may differentiate under the finite trajectory
expectation. Therefore,
\begin{align*}
    \nabla_\theta^2 J(\theta) &= \bE_{\pi_\theta}\sqbr{\sum_{t=1}^H Q_t^{\pi_\theta}(O_{1:t},A_{0:t},S_t) B_t(\theta)} +
    \bE_{\pi_\theta}\sqbr{\sum_{t=1}^H Z_t(\theta) \nabla_\theta Q_t^{\pi_\theta}(O_{1:t},A_{0:t},S_t)^\top} \\
    &\quad + \bE_{\pi_\theta}\sqbr{\sum_{t=1}^{H}{Q_t^{\pi_\theta}(O_{1:t},A_{0:t},S_t) Z_t(\theta) \br{\sum_{j=1}^t Z_j(\theta)}^\top}}.
\end{align*}
Here and below, gradients are viewed as column vectors. If the opposite Jacobian convention is used, the displayed matrices are transposed, which does not affect the spectral-norm bounds.

Taking spectral norms and using the triangle inequality together with $\|uv^\top\|_2=\|u\|_2\|v\|_2$, we get
\[
\begin{aligned}
    &\left\|\nabla_\theta^2 J(\theta)\right\|_2 \\
    &\leq \bE_{\pi_\theta}\sqbr{\sum_{t=1}^H \left|Q_t^{\pi_\theta}(O_{1:t},A_{0:t},S_t)\right| \|B_t(\theta)\|_2}
     + \bE_{\pi_\theta}\sqbr{\sum_{t=1}^H \|Z_t(\theta)\|_2 \left\|\nabla_\theta Q_t^{\pi_\theta}(O_{1:t},A_{0:t},S_t)\right\|_2} \\
    &\quad + \bE_{\pi_\theta}\sqbr{\sum_{t=1}^H \left| Q_t^{\pi_\theta}(O_{1:t},A_{0:t},S_t) \right| \|Z_t(\theta)\|_2 \sum_{j=1}^t \|Z_j(\theta)\|_2}.
\end{aligned}
\]

We next bound each term on the right-hand side. Since each reward is bounded by $r_{\max}$ in
absolute value,
\[
    \left|
        Q_t^{\pi_\theta}(O_{1:t},A_{0:t},S_t)
    \right|
    \leq
    r_{\max}(H-t+1).
\]
Therefore the first term is bounded by
\[
    T_1
    :=
    \sum_{t=1}^H
    r_{\max}(H-t+1)M
    =
    r_{\max}M\frac{H(H+1)}{2}.
\]

It remains to bound $\nabla_\theta Q_t^{\pi_\theta}$. For $t=H$,
\[
    Q_H^{\pi_\theta}(o_{1:H},a_{0:H},s_H)
    =
    r_H(o_{1:H},a_{1:H}),
\]
which does not depend on $\theta$, and hence
\[
    \nabla_\theta Q_H^{\pi_\theta}(o_{1:H},a_{0:H},s_H)=0.
\]
For $t\leq H-1$, by the definitions of $Q_t^\pi$ and $V_{t+1}^\pi$,
\[
    Q_t^{\pi_\theta}(o_{1:t},a_{0:t},s_t)
    =
    r_t(o_{1:t},a_{1:t})
    +
    \sum_{o_{t+1}\in\mathcal{O}}
    p_t(o_{t+1}\mid o_{1:t},a_{1:t})
    V_{t+1}^{\pi_\theta}(o_{1:t+1},a_{0:t},s_t).
\]
The reward function and the environment kernel do not depend on $\theta$. Thus,
\[
    \nabla_\theta
    Q_t^{\pi_\theta}(o_{1:t},a_{0:t},s_t)
    =
    \sum_{o_{t+1}\in\mathcal{O}}
    p_t(o_{t+1}\mid o_{1:t},a_{1:t})
    \nabla_\theta
    V_{t+1}^{\pi_\theta}(o_{1:t+1},a_{0:t},s_t).
\]
Applying the $Q$-function form of Proposition~\ref{prop:fhpg} to $V_{t+1}^{\pi_\theta}$ gives
\[
\begin{aligned}
    \nabla_\theta
    Q_t^{\pi_\theta}(o_{1:t},a_{0:t},s_t)
    =
    \bE_{\pi_\theta}
    \Bigg[
        \sum_{\ell=t+1}^H
        Q_\ell^{\pi_\theta}(O_{1:\ell},A_{0:\ell},S_\ell)
        Z_\ell(\theta)
        \,\Bigg|\,
        O_{1:t}=o_{1:t},A_{0:t}=a_{0:t},S_t=s_t
    \Bigg].
\end{aligned}
\]
Consequently,
\[
\begin{aligned}
    \left\|
        \nabla_\theta
        Q_t^{\pi_\theta}(o_{1:t},a_{0:t},s_t)
    \right\|_2
    &\leq
    \bE_{\pi_\theta}
    \Bigg[
        \sum_{\ell=t+1}^H
        \left|
            Q_\ell^{\pi_\theta}(O_{1:\ell},A_{0:\ell},S_\ell)
        \right|
        \|Z_\ell(\theta)\|_2
        \,\Bigg|\,
        O_{1:t}=o_{1:t},A_{0:t}=a_{0:t},S_t=s_t
    \Bigg]                                                        \\
    &\leq
    \sum_{\ell=t+1}^H
    r_{\max}(H-\ell+1)G.
\end{aligned}
\]
The same bound also holds for $t=H$, with the convention that the empty sum is zero. Hence, the second term in the Hessian bound is at most
\[
\begin{aligned}
    T_2
    &:=
    \sum_{t=1}^H
    G
    \sum_{\ell=t+1}^H
    r_{\max}(H-\ell+1)G                                            \\
    &=
    r_{\max}G^2
    \sum_{t=1}^H
    \sum_{\ell=t+1}^H
    (H-\ell+1).
\end{aligned}
\]
We compute
\[
\begin{aligned}
    \sum_{t=1}^H
    \sum_{\ell=t+1}^H
    (H-\ell+1)
    &=
    \sum_{\ell=1}^H
    (\ell-1)(H-\ell+1)                                             \\
    &=
    \frac{H(H-1)(H+1)}{6}.
\end{aligned}
\]
Therefore,
\[
    T_2
    =
    r_{\max}G^2
    \frac{H(H-1)(H+1)}{6}.
\]

Finally, the third term is bounded by
\[
\begin{aligned}
    T_3
    &:=
    \sum_{t=1}^H
    r_{\max}(H-t+1)G
    \sum_{j=1}^t G                                                  \\
    &=
    r_{\max}G^2
    \sum_{t=1}^H
    t(H-t+1).
\end{aligned}
\]
Using
\[
    \sum_{t=1}^H
    t(H-t+1)
    =
    \frac{H(H+1)(H+2)}{6},
\]
we obtain
\[
    T_3
    =
    r_{\max}G^2
    \frac{H(H+1)(H+2)}{6}.
\]

Combining the three bounds,
\[
\begin{aligned}
    \left\|\nabla_\theta^2 J(\theta)\right\|_2
    &\leq
    T_1+T_2+T_3                                                     \\
    &=
    r_{\max}M\frac{H(H+1)}{2}
    +
    r_{\max}G^2\frac{H(H-1)(H+1)}{6}
    +
    r_{\max}G^2\frac{H(H+1)(H+2)}{6}                                \\
    &=
    \frac{r_{\max}H(H+1)}{6}
    \left(
        3M
        +
        G^2\bigl((H-1)+(H+2)\bigr)
    \right)                                                        \\
    &=
    \frac{r_{\max}H(H+1)}{6}
    \left(
        3M+G^2(2H+1)
    \right).
\end{aligned}
\]
Thus,
\[
    \left\|\nabla_\theta^2 J(\theta)\right\|_2
    \leq
    \beta,
    \qquad
    \beta
    :=
    \frac{r_{\max}H(H+1)}{6}
    \left(3M+G^2(2H+1)\right).
\]
This proves that $J$ is $\beta$-smooth.
\end{proof}

\subsection{Convergence results}
\label{subapp:proof_conv_episodic}
\begin{theorem}[Extended version of Theorem~\ref{thm:conv_fhpg}]
Consider the iterates $\{\theta_k\}_{k\geq 1}$ generated by Algorithm~\ref{assum:smooth_parame_fh}. Suppose that the
episodic ASM policy parameterization satisfies the regularity conditions of Theorem~\ref{thm:fhpg} and
Assumption~\ref{assum:smooth_parame_fh}. Let
\[
    \Delta_1 := \sup_{\theta\in\mathbb{R}^d} J(\theta) - J(\theta_1),
\]
and let $\beta$ be the smoothness constant from Lemma~\ref{lem:finite_smooth}. Define
\[
    C_H := r_{\max} G H(H+1).
\]
Then the following statements hold.

\begin{enumerate}
    \item[(a)] Fix $K\geq 1$. If
    \[
        \alpha_k = \alpha
        :=
        \min\left\{
            \frac{1}{\beta},
            \frac{2}{C_H}\sqrt{\frac{2\Delta_1}{\beta K}}
        \right\},
        \qquad k=1,\ldots,K,
    \]
    then
    \[
        \frac{1}{K}\sum_{k=1}^K
        \bE\left[
            \left\|\nabla_\theta J(\theta_k)\right\|_2^2
        \right]
        \leq
        \frac{2\beta\Delta_1}{K}
        +
        C_H\sqrt{\frac{\beta\Delta_1}{2K}}.
    \]

    \item[(b)] Fix $K\geq 1$ and $\delta\in(0,1)$. With probability at least $1-\delta$,
    the following bounds hold for the corresponding choices of stepsizes.

    \begin{enumerate}
        \item[(i)] If
        \[
            \alpha_k = \alpha
            :=
            \min\left\{
                \frac{1}{\beta},
                \frac{1}{C_H}\sqrt{\frac{\Delta_1}{\beta K}}
            \right\},
            \qquad k=1,\ldots,K,
        \]
        then
        \[
            \frac{1}{K}\sum_{k=1}^K
            \left\|\nabla_\theta J(\theta_k)\right\|_2^2
            \leq
            \frac{2\beta\Delta_1}{K}
            +
            5C_H\sqrt{\frac{\beta\Delta_1}{K}}
            +
            \frac{12C_H^2\log(1/\delta)}{K}.
        \]

        \item[(ii)] If
        \[
            \alpha_k = \frac{1}{\beta\sqrt{k}},
            \qquad k=1,\ldots,K,
        \]
        then
        \[
            \frac{1}{K}\sum_{k=1}^K
            \left\|\nabla_\theta J(\theta_k)\right\|_2^2
            \leq
            \frac{
                2\beta\Delta_1
                +
                3C_H^2(1+\log K)
                +
                12C_H^2\log(1/\delta)
            }{\sqrt{K}}.
        \]
    \end{enumerate}

    \item[(c)] If $\{\alpha_k\}_{k\geq 1}$ is any deterministic positive stepsize sequence satisfying
    \[
        \sum_{k=1}^{\infty}\alpha_k=\infty,
        \qquad
        \sum_{k=1}^{\infty}\alpha_k^2<\infty,
    \]
    then
    \[
        \left\|\nabla_\theta J(\theta_k)\right\|_2 \longrightarrow 0
        \qquad\text{almost surely.}
    \]
\end{enumerate}
\end{theorem}

\begin{proof}
Let $\tau_k := (O^k_{1:H},S^k_{0:H},A^k_{0:H})$ denote the trajectory generated during
episode $k$, and let
\[
    \mathcal{G}_k := \sigma(\tau_1,\ldots,\tau_{k-1})
\]
be the sigma-field generated by all trajectories sampled before episode $k$. In particular,
$\theta_k$ is $\mathcal{G}_k$-measurable.

Let $\widehat g_k$ denote the stochastic gradient estimator computed in Algorithm~\ref{assum:smooth_parame_fh}, namely
\[
    \widehat g_k
    =
    \sum_{t=1}^H
    \left(
        \sum_{\ell=t}^H
        r_\ell(O^k_{1:\ell},A^k_{1:\ell})
    \right)
    \nabla_\theta
    \log
    \pi_{\theta_k,t}
    \left(
        S_t^k,A_t^k
        \mid
        S_{t-1}^k,A_{t-1}^k,O_t^k
    \right).
\]
By Theorem~\ref{thm:fhpg}, conditionally on $\mathcal{G}_k$, the estimator is unbiased:
\[
    \bE\left[\widehat g_k\mid \mathcal{G}_k\right]
    =
    \nabla_\theta J(\theta_k).
\]
For notational simplicity, write
\[
    g_k := \nabla_\theta J(\theta_k),
    \qquad
    \xi_k := \widehat g_k - g_k.
\]
Then
\[
    \bE\left[\xi_k\mid \mathcal{G}_k\right]=0.
\]

We first record deterministic bounds on the estimator and on the noise. Since
$|r_t|\leq r_{\max}$ and Assumption~\ref{assum:smooth_parame_fh} gives
\[
    \left\|
    \nabla_\theta
    \log
    \pi_{\theta_k,t}
    \left(
        S_t^k,A_t^k
        \mid
        S_{t-1}^k,A_{t-1}^k,O_t^k
    \right)
    \right\|_2
    \leq G,
\]
we have, almost surely,
\[
\begin{aligned}
    \left\|\widehat g_k\right\|_2
    &\leq
    \sum_{t=1}^H
    \left|
        \sum_{\ell=t}^H
        r_\ell(O^k_{1:\ell},A^k_{1:\ell})
    \right|
    \left\|
    \nabla_\theta
    \log
    \pi_{\theta_k,t}
    \left(
        S_t^k,A_t^k
        \mid
        S_{t-1}^k,A_{t-1}^k,O_t^k
    \right)
    \right\|_2                                                        \\
    &\leq
    \sum_{t=1}^H r_{\max}(H-t+1)G                                      \\
    &=
    \frac{r_{\max}GH(H+1)}{2}
    =
    \frac{C_H}{2}.
\end{aligned}
\]
Consequently, by Jensen's inequality and conditional unbiasedness,
\[
    \left\|g_k\right\|_2
    =
    \left\|
        \bE\left[\widehat g_k\mid \mathcal{G}_k\right]
    \right\|_2
    \leq
    \bE\left[
        \left\|\widehat g_k\right\|_2
        \mid \mathcal{G}_k
    \right]
    \leq
    \frac{C_H}{2}.
\]
Therefore, almost surely,
\[
    \left\|\xi_k\right\|_2
    \leq
    \left\|\widehat g_k\right\|_2+\left\|g_k\right\|_2
    \leq C_H.
\]
Moreover,
\[
\begin{aligned}
    \bE\left[
        \left\|\xi_k\right\|_2^2
        \mid \mathcal{G}_k
    \right]
    &=
    \bE\left[
        \left\|\widehat g_k-g_k\right\|_2^2
        \mid \mathcal{G}_k
    \right]                                                        \\
    &=
    \bE\left[
        \left\|\widehat g_k\right\|_2^2
        \mid \mathcal{G}_k
    \right]
    -
    \left\|g_k\right\|_2^2                                        \leq
    \left(\frac{C_H}{2}\right)^2.
\end{aligned}
\]

Since $J$ is $\beta$-smooth by Lemma~\ref{lem:finite_smooth}, for every $k$,
\[
    J(\theta_{k+1})
    \geq
    J(\theta_k)
    +
    \left\langle
        g_k,\theta_{k+1}-\theta_k
    \right\rangle
    -
    \frac{\beta}{2}
    \left\|
        \theta_{k+1}-\theta_k
    \right\|_2^2.
\]
Using the ASMPG update $\theta_{k+1}=\theta_k+\alpha_k\widehat g_k$ and
$\widehat g_k=g_k+\xi_k$, this gives
\[
\begin{aligned}
    J(\theta_{k+1})
    &\geq
    J(\theta_k)
    +
    \alpha_k
    \left\langle
        g_k,\widehat g_k
    \right\rangle
    -
    \frac{\beta\alpha_k^2}{2}
    \left\|\widehat g_k\right\|_2^2                                  \\
    &=
    J(\theta_k)
    +
    \alpha_k\left\|g_k\right\|_2^2
    +
    \alpha_k
    \left\langle
        g_k,\xi_k
    \right\rangle
    -
    \frac{\beta\alpha_k^2}{2}
    \left\|\widehat g_k\right\|_2^2                                  \\
    &\geq
    J(\theta_k)
    +
    \alpha_k\left\|g_k\right\|_2^2
    +
    \alpha_k
    \left\langle
        g_k,\xi_k
    \right\rangle
    -
    \frac{\beta C_H^2}{8}\alpha_k^2.
\end{aligned}
\]
Equivalently,
\[
    \alpha_k\left\|g_k\right\|_2^2
    \leq
    J(\theta_{k+1})-J(\theta_k)
    -
    \alpha_k
    \left\langle
        g_k,\xi_k
    \right\rangle
    +
    \frac{\beta C_H^2}{8}\alpha_k^2.
\]
Summing from $k=1$ to $K$ and using
$J(\theta_{K+1})\leq \sup_{\theta\in\mathbb{R}^d}J(\theta)$, we obtain
\[
    \sum_{k=1}^K
    \alpha_k\left\|g_k\right\|_2^2
    \leq
    \Delta_1
    -
    \sum_{k=1}^K
    \alpha_k
    \left\langle
        g_k,\xi_k
    \right\rangle
    +
    \frac{\beta C_H^2}{8}
    \sum_{k=1}^K\alpha_k^2.
\]

We now prove part~(a). Taking expectations in the preceding inequality, the martingale
difference term vanishes because $g_k$ is $\mathcal{G}_k$-measurable and
$\bE[\xi_k\mid\mathcal{G}_k]=0$. Hence, for every deterministic stepsize sequence,
\[
    \sum_{k=1}^K
    \alpha_k\,
    \bE
    \left[
        \left\|g_k\right\|_2^2
    \right]
    \leq
    \Delta_1
    +
    \frac{\beta C_H^2}{8}
    \sum_{k=1}^K\alpha_k^2.
\]
For the constant stepsize in part~(a), this implies
\[
    \frac{1}{K}\sum_{k=1}^K
    \bE
    \left[
        \left\|g_k\right\|_2^2
    \right]
    \leq
    \frac{\Delta_1}{\alpha K}
    +
    \frac{\beta C_H^2}{8}\alpha.
\]
If $\Delta_1=0$, then $\theta_1$ is a global maximizer of the differentiable function $J$,
so $g_1=0$; with the displayed constant stepsize the iterates do not move, and the claim is
immediate. Thus assume $\Delta_1>0$. Define
\[
    \overline{\alpha}
    :=
    \frac{2}{C_H}
    \sqrt{\frac{2\Delta_1}{\beta K}},
    \qquad
    \alpha
    =
    \min\left\{
        \frac{1}{\beta},
        \overline{\alpha}
    \right\}.
\]
Then
\[
    \frac{1}{\alpha}
    \leq
    \beta + \frac{1}{\overline{\alpha}}
    =
    \beta
    +
    \frac{C_H}{2}
    \sqrt{\frac{\beta K}{2\Delta_1}},
\]
and $\alpha\leq\overline{\alpha}$. Therefore,
\[
\begin{aligned}
    \frac{1}{K}\sum_{k=1}^K
    \bE
    \left[
        \left\|g_k\right\|_2^2
    \right]
    &\leq
    \frac{\Delta_1}{K}
    \left(
        \beta
        +
        \frac{C_H}{2}
        \sqrt{\frac{\beta K}{2\Delta_1}}
    \right)
    +
    \frac{\beta C_H^2}{8}
    \overline{\alpha}                                                \\
    &=
    \frac{\beta\Delta_1}{K}
    +
    \frac{C_H}{2}
    \sqrt{\frac{\beta\Delta_1}{2K}}
    +
    \frac{C_H}{2}
    \sqrt{\frac{\beta\Delta_1}{2K}}                                  \\
    &=
    \frac{\beta\Delta_1}{K}
    +
    C_H
    \sqrt{\frac{\beta\Delta_1}{2K}}                                  \\
    &\leq
    \frac{2\beta\Delta_1}{K}
    +
    C_H
    \sqrt{\frac{\beta\Delta_1}{2K}}.
\end{aligned}
\]
This proves part~(a).

We next prove the high-probability bounds in part~(b). Fix $K\geq 1$ and let
\[
    \eta_K := \max_{1\leq k\leq K}\alpha_k.
\]
Define
\[
    S_K :=
    \sum_{k=1}^K
    \alpha_k\left\|g_k\right\|_2^2,
    \qquad
    A_K :=
    \sum_{k=1}^K\alpha_k^2,
    \qquad
    Z_K :=
    -
    \sum_{k=1}^K
    \alpha_k
    \left\langle
        g_k,\xi_k
    \right\rangle.
\]
The preceding descent inequality can be written as
\[
    S_K
    \leq
    \Delta_1 + Z_K + \frac{\beta C_H^2}{8}A_K.
\]

We now bound $Z_K$. Conditionally on $\mathcal{G}_k$, the scalar random variable
$\langle g_k,\xi_k\rangle$ has mean zero and satisfies
\[
    \left|\left\langle g_k,\xi_k\right\rangle\right|
    \leq
    C_H\left\|g_k\right\|_2.
\]
Hence, by Hoeffding's lemma, for every $\lambda>0$,
\[
    \bE
    \left[
        \exp\left(
            -\lambda\alpha_k
            \left\langle
                g_k,\xi_k
            \right\rangle
        \right)
        \mid \mathcal{G}_k
    \right]
    \leq
    \exp\left(
        \frac{\lambda^2\alpha_k^2 C_H^2}{2}
        \left\|g_k\right\|_2^2
    \right).
\]
Iterating conditional expectations gives
\[
    \bE
    \left[
        \exp\left(
            \lambda Z_K
            -
            \frac{\lambda^2 C_H^2}{2}
            \sum_{k=1}^K
            \alpha_k^2
            \left\|g_k\right\|_2^2
        \right)
    \right]
    \leq 1.
\]
Therefore, by Markov's inequality, with probability at least $1-\delta$,
\[
    Z_K
    \leq
    \frac{\lambda C_H^2}{2}
    \sum_{k=1}^K
    \alpha_k^2
    \left\|g_k\right\|_2^2
    +
    \frac{\log(1/\delta)}{\lambda}.
\]
Since $\alpha_k^2\leq \eta_K\alpha_k$ for all $k\leq K$, we have
\[
    Z_K
    \leq
    \frac{\lambda C_H^2\eta_K}{2}S_K
    +
    \frac{\log(1/\delta)}{\lambda}.
\]
Choosing $\lambda=(C_H^2\eta_K)^{-1}$ yields
\[
    Z_K
    \leq
    \frac{1}{2}S_K
    +
    C_H^2\eta_K\log(1/\delta).
\]
Combining this estimate with the bound on $S_K$, we obtain, with probability at least $1-\delta$,
\[
    S_K
    \leq
    2\Delta_1
    +
    \frac{\beta C_H^2}{4}A_K
    +
    2C_H^2\eta_K\log(1/\delta).
\]

We first apply this bound to the constant stepsize in part~(b)(i). Let
\[
    \overline{\alpha}
    :=
    \frac{1}{C_H}
    \sqrt{\frac{\Delta_1}{\beta K}},
    \qquad
    \alpha
    =
    \min\left\{
        \frac{1}{\beta},
        \overline{\alpha}
    \right\}.
\]
If $\Delta_1=0$, the claim is immediate as above. Thus assume $\Delta_1>0$.
For constant stepsize, $\eta_K=\alpha$, $A_K=K\alpha^2$, and $S_K=\alpha
\sum_{k=1}^K\|g_k\|_2^2$. Hence, with probability at least $1-\delta$,
\[
    \frac{1}{K}\sum_{k=1}^K
    \left\|g_k\right\|_2^2
    \leq
    \frac{2\Delta_1}{\alpha K}
    +
    \frac{\beta C_H^2}{4}\alpha
    +
    \frac{2C_H^2\log(1/\delta)}{K}.
\]
Moreover,
\[
    \frac{1}{\alpha}
    \leq
    \beta + \frac{1}{\overline{\alpha}}
    =
    \beta + C_H\sqrt{\frac{\beta K}{\Delta_1}},
    \qquad
    \alpha\leq\overline{\alpha}.
\]
Therefore,
\[
\begin{aligned}
    \frac{1}{K}\sum_{k=1}^K
    \left\|g_k\right\|_2^2
    &\leq
    \frac{2\beta\Delta_1}{K}
    +
    2C_H\sqrt{\frac{\beta\Delta_1}{K}}
    +
    \frac{C_H}{4}\sqrt{\frac{\beta\Delta_1}{K}}
    +
    \frac{2C_H^2\log(1/\delta)}{K}                                  \\
    &\leq
    \frac{2\beta\Delta_1}{K}
    +
    5C_H\sqrt{\frac{\beta\Delta_1}{K}}
    +
    \frac{12C_H^2\log(1/\delta)}{K}.
\end{aligned}
\]
This proves part~(b)(i).

We now prove part~(b)(ii). Let $\alpha_k=1/(\beta\sqrt{k})$ for $k=1,\ldots,K$.
Then
\[
    \eta_K = \frac{1}{\beta},
    \qquad
    A_K
    =
    \frac{1}{\beta^2}
    \sum_{k=1}^K \frac{1}{k}
    \leq
    \frac{1+\log K}{\beta^2}.
\]
The high-probability bound on $S_K$ gives
\[
    S_K
    \leq
    2\Delta_1
    +
    \frac{C_H^2}{4\beta}(1+\log K)
    +
    \frac{2C_H^2}{\beta}\log(1/\delta).
\]
Also, since $\alpha_k\geq \alpha_K=1/(\beta\sqrt{K})$ for $k=1,\ldots,K$,
\[
    \frac{1}{\beta\sqrt{K}}
    \sum_{k=1}^K
    \left\|g_k\right\|_2^2
    \leq
    S_K.
\]
Consequently, with probability at least $1-\delta$,
\[
\begin{aligned}
    \frac{1}{K}\sum_{k=1}^K
    \left\|g_k\right\|_2^2
    &\leq
    \frac{\beta S_K}{\sqrt{K}}                                      \\
    &\leq
    \frac{
        2\beta\Delta_1
        +
        \frac{C_H^2}{4}(1+\log K)
        +
        2C_H^2\log(1/\delta)
    }{\sqrt{K}}                                                     \\
    &\leq
    \frac{
        2\beta\Delta_1
        +
        3C_H^2(1+\log K)
        +
        12C_H^2\log(1/\delta)
    }{\sqrt{K}}.
\end{aligned}
\]
This proves part~(b)(ii).

It remains to prove part~(c). Assume that
\[
    \sum_{k=1}^{\infty}\alpha_k=\infty,
    \qquad
    \sum_{k=1}^{\infty}\alpha_k^2<\infty.
\]
Taking conditional expectation in the smoothness inequality gives
\[
    \bE
    \left[
        J(\theta_{k+1})
        \mid \mathcal{G}_k
    \right]
    \geq
    J(\theta_k)
    +
    \alpha_k\left\|g_k\right\|_2^2
    -
    \frac{\beta C_H^2}{8}\alpha_k^2.
\]
Let
\[
    J^\star := \sup_{\theta\in\mathbb{R}^d}J(\theta),
    \qquad
    Y_k := J^\star - J(\theta_k).
\]
Then $Y_k\geq 0$ and
\[
    \bE
    \left[
        Y_{k+1}
        \mid \mathcal{G}_k
    \right]
    \leq
    Y_k
    -
    \alpha_k\left\|g_k\right\|_2^2
    +
    \frac{\beta C_H^2}{8}\alpha_k^2.
\]
Since $\sum_{k=1}^{\infty}\alpha_k^2<\infty$, the Robbins--Siegmund almost
supermartingale theorem implies
\[
    \sum_{k=1}^{\infty}
    \alpha_k\left\|g_k\right\|_2^2
    <\infty
    \qquad\text{almost surely.}
\]

We now show that this implies $\|g_k\|_2\to 0$ almost surely. Since $J$ is
$\beta$-smooth,
\[
    \left\|g_{k+1}-g_k\right\|_2
    =
    \left\|
        \nabla_\theta J(\theta_{k+1})
        -
        \nabla_\theta J(\theta_k)
    \right\|_2
    \leq
    \beta
    \left\|
        \theta_{k+1}-\theta_k
    \right\|_2.
\]
Using $\theta_{k+1}-\theta_k=\alpha_k\widehat g_k$ and
$\|\widehat g_k\|_2\leq C_H/2$, we get
\[
    \left\|g_{k+1}-g_k\right\|_2
    \leq
    \frac{\beta C_H}{2}\alpha_k.
\]
Let
\[
    v_k := \left\|g_k\right\|_2^2.
\]
Since $\|g_k\|_2\leq C_H/2$ for all $k$, we have
\[
\begin{aligned}
    |v_{k+1}-v_k|
    &=
    \left|
        \left\|g_{k+1}\right\|_2^2
        -
        \left\|g_k\right\|_2^2
    \right|                                                         \\
    &\leq
    \left(
        \left\|g_{k+1}\right\|_2
        +
        \left\|g_k\right\|_2
    \right)
    \left\|g_{k+1}-g_k\right\|_2                                    \\
    &\leq
    \frac{\beta C_H^2}{2}\alpha_k.
\end{aligned}
\]
Set $L:=\beta C_H^2/2$. Thus
\[
    |v_{k+1}-v_k|\leq L\alpha_k.
\]
We already know that
\[
    \sum_{k=1}^{\infty}\alpha_k v_k < \infty
    \qquad\text{almost surely.}
\]
Fix an outcome on which this finite-sum property holds. Suppose, for contradiction, that
$v_k$ does not converge to $0$. Then there exists $\varepsilon>0$ such that
$v_k\geq\varepsilon$ for infinitely many $k$. Since
$\sum_k\alpha_k=\infty$ while $\sum_k\alpha_k v_k<\infty$, the sequence $v_k$ cannot be
eventually larger than or equal to $\varepsilon/2$. Hence there are infinitely many
upcrossings from below $\varepsilon/2$ to above $\varepsilon$.

Choose one such upcrossing interval $[r,q]$, with $r<q$, such that
\[
    v_r < \frac{\varepsilon}{2},
    \qquad
    v_q \geq \varepsilon,
\]
and $r$ is the last index before $q$ for which $v_r<\varepsilon/2$. Then
$v_j\geq \varepsilon/2$ for all $j=r+1,\ldots,q$. Also,
\[
    \frac{\varepsilon}{2}
    \leq
    v_q-v_r
    \leq
    \sum_{j=r}^{q-1}|v_{j+1}-v_j|
    \leq
    L\sum_{j=r}^{q-1}\alpha_j.
\]
Because $\sum_k\alpha_k^2<\infty$, we have $\alpha_k\to 0$. Hence, for all sufficiently
large upcrossings,
\[
    \alpha_r \leq \frac{\varepsilon}{4L},
\]
and therefore
\[
    \sum_{j=r+1}^{q-1}\alpha_j
    \geq
    \frac{\varepsilon}{4L}.
\]
On the same indices $j=r+1,\ldots,q-1$, we have $v_j\geq \varepsilon/2$. Thus each
sufficiently late upcrossing contributes at least
\[
    \sum_{j=r+1}^{q-1}\alpha_j v_j
    \geq
    \frac{\varepsilon}{2}
    \sum_{j=r+1}^{q-1}\alpha_j
    \geq
    \frac{\varepsilon^2}{8L}
\]
to the sum $\sum_k\alpha_k v_k$. Infinitely many disjoint upcrossings would therefore force
$\sum_k\alpha_k v_k=\infty$, contradicting the finite-sum property. Hence
$v_k\to 0$ almost surely, i.e.,
\[
    \left\|\nabla_\theta J(\theta_k)\right\|_2
    =
    \left\|g_k\right\|_2
    \longrightarrow 0
    \qquad\text{almost surely.}
\]
This proves part~(c), and completes the proof.
\end{proof}

\section{Results related to infinite-horizon discounted NMDPs}
\label{app:ihd_nmdp}
\subsection{Policy gradient theorem}
\label{subapp:pg_ihd}
Recall the definition of the discounted reward $Q$-functions for an ASM policy $\pi$,
\begin{align*}
    Q^{\gamma,\pi}_t(o_{1:t},a_{0:t},s_t)
    = \bE_\pi\sqbr{R^\gamma_{t:\infty} \mid O_{1:t}=o_{1:t}, A_{0:t}=a_{0:t}, S_t=s_t},
\end{align*}
where $t\ge 1$, $R^\gamma_{t:\infty} = \sum_{t\up=t}^{\infty}{\gamma^{t\up - t} r_{t\up}(O_{1:t\up}, A_{1:t\up})}$ and $\bE_\pi$ denotes the expectation taken considering that the agent states and actions from time $t$ onward are generated by $\pi$. Likewise, define the discounted-reward value function for a stationary ASM policy $\pi$, as follows:
\begin{align*}
    V^{\gamma,\pi}_t(o_{1:t}, a_{0:t-1}, s_{t-1}) \coloneqq \bE_\pi\sqbr{R^\gamma_{t:\infty} \mid O_{1:t} = o_{1:t}, A_{0:t-1} = a_{0:t-1}, S_{t-1} = s_{t-1}}.
\end{align*}
for $t \geq 1$. The next proposition provides expressions for the gradient of the value functions with respect to policy parameters.

\begin{proposition}
    \label{prop:ih_value_grad_t}
    Let $\{\pi_\theta \mid \theta \in \bR^d\}$ be a class of parameterized stationary ASM policies that are differentiable with respect to the parameter vector $\theta$, and the support of $\pi_\theta$ is independent of $\theta$. Suppose that there exists $r_{\max} > 0$ such that $|r_t(\cdot)| \le r_{\max}$ for every $t \ge 1$, and there exists $G > 0$ such that
    \begin{align*}
        \norm{\nabla_\theta \log\br{\pi_\theta(s,a \mid \tilde{s},\tilde{a},o)}}_2 \le G
    \end{align*}
    for every $s,\tilde{s} \in \cS$, $a,\tilde{a} \in \cA$, $o \in \cO$, and $\theta \in \bR^d$. Then, for every $t \ge 1$, $o_{1:t} \in \cO^t, a_{0:t-1} \in \cA^t$, and $s_{t-1} \in \cS$,
    \begin{align}
        &\nabla_\theta{V^{\gamma,\pi_\theta}_t(o_{1:t}, a_{0:t-1}, s_{t-1})} \notag\\
        &= \bE_{\pi_\theta}\sqbr{\sum_{t\up=t}^{\infty}{\gamma^{t\up-t} R^\gamma_{t\up:\infty} \nabla_\theta \log\br{\pi_\theta(S_{t\up}, A_{t\up} \mid S_{t\up-1}, A_{t\up-1}, O_{t\up})}} \mid o_{1:t}, a_{0:t-1}, s_{t-1}} \label{pg:ih_1}\\
        &= \bE_{\pi_\theta}\sqbr{\sum_{t\up=t}^{\infty}{\gamma^{t\up-t} Q^{\gamma,\pi_\theta}_{t\up}(O_{1:t\up}, A_{0:t\up}, S_{t\up}) \nabla_\theta \log\br{\pi_\theta(S_{t\up}, A_{t\up} \mid S_{t\up-1}, A_{t\up-1}, O_{t\up})}} \mid o_{1:t}, a_{0:t-1}, s_{t-1}}. \label{pg:ih_2}
    \end{align}
\end{proposition}

\begin{proof}
    Fix $t \ge 1$ and let
    \begin{align*}
        h_t \coloneqq (o_{1:t}, a_{0:t-1}, s_{t-1}).
    \end{align*}
    For $T \ge t$, define the truncated return and the truncated value function as
    \begin{align*}
        V^{\gamma,\pi_\theta}_{t,T}(h_t) \coloneqq \bE_{\pi_\theta}\sqbr{R^\gamma_{t:T} \mid h_t},
    \end{align*}
    where $R^\gamma_{t:T} = \sum_{\ell=t}^{T}{\gamma^{\ell-t}r_\ell(O_{1:\ell},A_{1:\ell})}$. We first prove the desired identity for $V^{\gamma,\pi_\theta}_{t,T}(h_t)$ and then let $T \to \infty$.
    
    Denote $\bP^{\pi_\theta}_T(\cdot \mid h_t)$ the conditional law of the finite-length trajectory $(O_{t+1:T}, A_{t:T}, S_{t:T})$ induced by $\pi_\theta$ given $h_t$. For a finite-length trajectory $\tau_T \coloneqq (o_{t+1:T}, a_{t:T}, s_{t:T})$,
    \begin{align*}
        \bP^{\pi_\theta}_T(\tau_T \mid h_t) = \br{\prod_{t\up=t}^{T}{\pi_\theta(s_{t\up}, a_{t\up} \mid s_{t\up-1}, a_{t\up-1}, o_{t\up})}}
        \br{\prod_{t\up=t}^{T-1}{p_{t\up}(o_{t\up+1} \mid o_{1:t\up}, a_{1:t\up})}},
    \end{align*}
    where the second product is interpreted as one when $T=t$. On the support of this finite-length trajectory law, its logarithm satisfies
    \begin{align*}
        \log\br{\bP^{\pi_\theta}_T(\tau_T \mid h_t)} &= \sum_{t\up=t}^{T}{\log\br{\pi_\theta(s_{t\up}, a_{t\up} \mid s_{t\up-1}, a_{t\up-1}, o_{t\up})}} + \sum_{t\up=t}^{T-1}{\log\br{p_{t\up}(o_{t\up+1} \mid o_{1:t\up}, a_{1:t\up})}},
    \end{align*}
    where only the first term depends on $\theta$. Thus,
    \begin{align*}
        \nabla_\theta \log\br{\bP^{\pi_\theta}_T(\tau_T \mid h_t)} = \sum_{t\up=t}^{T}{\nabla_\theta \log\br{\pi_\theta(s_{t\up}, a_{t\up} \mid s_{t\up-1}, a_{t\up-1}, o_{t\up})}}.
    \end{align*}
    Since the support of $\pi_\theta$ is independent of $\theta$, the support of $\bP^{\pi_\theta}_T(\cdot \mid h_t)$ is also independent of $\theta$. Hence, applying the log-derivative trick to the finite-length trajectory $\tau_T$, we have
    \begin{align}
        \nabla_\theta V^{\gamma,\pi_\theta}_{t,T}(h_t) = \bE_{\pi_\theta}\sqbr{R^\gamma_{t:T}\sum_{t\up=t}^{T}{\nabla_\theta \log\br{\pi_\theta(S_{t\up}, A_{t\up} \mid S_{t\up-1}, A_{t\up-1}, O_{t\up})}} \mid h_t}.
        \label{eq:ih_finite_log_derivative}
    \end{align}

    Note that, for every $t\up \in \{t,\ldots,T\}$,
    \begin{align*}
        R^\gamma_{t:T} = R^\gamma_{t:t\up-1} + \gamma^{t\up-t}R^\gamma_{t\up:T}.
    \end{align*}
    Denote the $\sigma$-field
    \begin{align*}
        \cF_{t\up}
        \coloneqq
        \sigma(O_1, S_1, A_1, O_2, S_2, A_2, \ldots, O_{t\up-1}, S_{t\up-1}, A_{t\up-1}, O_{t\up}).
    \end{align*}
    Note that $R^\gamma_{t:t\up-1}$ is $\cF_{t\up}$-measurable. For every $t\up \in \{t,\ldots,T\}$, using the tower property of conditional expectation, we have
    \begin{align*}
        &\bE_{\pi_\theta}\sqbr{R^\gamma_{t:t\up-1} \nabla_\theta \log\br{\pi_\theta(S_{t\up}, A_{t\up} \mid S_{t\up-1}, A_{t\up-1}, O_{t\up})} \mid h_t} \\
        &= \bE_{\pi_\theta}\sqbr{R^\gamma_{t:t\up-1} \bE_{\pi_\theta}\sqbr{\nabla_\theta \log\br{\pi_\theta(S_{t\up}, A_{t\up} \mid S_{t\up-1}, A_{t\up-1}, O_{t\up})} \mid \cF_{t\up}} \mid h_t} \\
        &= \bE_{\pi_\theta}\Bigg[R^\gamma_{t:t\up-1} \nabla_\theta \sum_{\substack{(s,a) \in \cS \times \cA:\\ \pi_\theta(s,a \mid S_{t\up-1}, A_{t\up-1}, O_{t\up})>0}}{\pi_\theta(s,a \mid S_{t\up-1}, A_{t\up-1}, O_{t\up})} \mid h_t\Bigg] \\
        &= 0.
    \end{align*}
    The last equality follows because the support of $\pi_\theta(\cdot,\cdot \mid S_{t\up-1}, A_{t\up-1}, O_{t\up})$ is independent of $\theta$, and the sum of the policy probabilities over this support is one. Summing the above identity over $t\up=t,\ldots,T$ and using \eqref{eq:ih_finite_log_derivative}, we obtain
    \begin{align}
        \nabla_\theta V^{\gamma,\pi_\theta}_{t,T}(h_t) = \bE_{\pi_\theta}\sqbr{\sum_{t\up=t}^{T}
        {\gamma^{t\up-t} R^\gamma_{t\up:T} \nabla_\theta \log\br{\pi_\theta(S_{t\up}, A_{t\up} \mid S_{t\up-1}, A_{t\up-1}, O_{t\up})}} \mid h_t}. \label{eq:ih_finite_reward_to_go}
    \end{align}

    It remains to pass to the limit $T \to \infty$. First, by boundedness of the rewards,
    \begin{align*}
        \abs{V^{\gamma,\pi_\theta}_{t}(h_t) - V^{\gamma,\pi_\theta}_{t,T}(h_t)} \leq \sum_{\ell=T+1}^{\infty}{\gamma^{\ell-t}r_{\max}} =
        \frac{r_{\max}\gamma^{T+1-t}}{1-\gamma}.
    \end{align*}
    Hence $V^{\gamma,\pi_\theta}_{t,T}(h_t) \to V^{\gamma,\pi_\theta}_{t}(h_t)$ uniformly in $\theta$.

    Next, for $L > T$, applying the same finite-length trajectory log-derivative argument to $V^{\gamma,\pi_\theta}_{t,L}(h_t)-V^{\gamma,\pi_\theta}_{t,T}(h_t)$ gives
    \begin{align*}
        \norm{\nabla_\theta V^{\gamma,\pi_\theta}_{t,L}(h_t) - \nabla_\theta V^{\gamma,\pi_\theta}_{t,T}(h_t)}_2 \leq G r_{\max}
        \sum_{\ell=T+1}^{L}{(\ell-t+1)\gamma^{\ell-t}}.
    \end{align*}
    The right-hand side goes to zero as $T \to \infty$, uniformly in $L$ and $\theta$. Therefore, $\{\nabla_\theta V^{\gamma,\pi_\theta}_{t,T}(h_t)\}_{T\ge t}$ is uniformly Cauchy. Since $V^{\gamma,\pi_\theta}_{t,T}(h_t)$ converges uniformly to $V^{\gamma,\pi_\theta}_{t}(h_t)$, the limit $V^{\gamma,\pi_\theta}_{t}(h_t)$ is differentiable and
    \begin{align*}
        \nabla_\theta V^{\gamma,\pi_\theta}_{t}(h_t) = \lim_{T\to\infty} \nabla_\theta V^{\gamma,\pi_\theta}_{t,T}(h_t).
    \end{align*}

    Finally, by the boundedness of the rewards and the score,
    \begin{align*}
        \norm{\sum_{t\up=t}^{T}{\gamma^{t\up-t} R^\gamma_{t\up:T}
        \nabla_\theta \log\br{\pi_\theta(S_{t\up}, A_{t\up} \mid S_{t\up-1}, A_{t\up-1}, O_{t\up})}}}_2 \leq \sum_{t\up=t}^{T}{\gamma^{t\up-t}\frac{r_{\max}}{1-\gamma}G} \leq \frac{G r_{\max}}{(1-\gamma)^2}.
    \end{align*}
    Therefore, by dominated convergence in \eqref{eq:ih_finite_reward_to_go}, we obtain
    \begin{align*}
        \nabla_\theta V^{\gamma,\pi_\theta}_{t}(h_t)
        =
        \bE_{\pi_\theta}\sqbr{
        \sum_{t\up=t}^{\infty}
        {\gamma^{t\up-t} R^\gamma_{t\up:\infty}
        \nabla_\theta \log\br{\pi_\theta(S_{t\up}, A_{t\up} \mid S_{t\up-1}, A_{t\up-1}, O_{t\up})}}
        \mid h_t}.
    \end{align*}
    This proves \eqref{pg:ih_1}. Now, we will prove~\eqref{pg:ih_2}. By definition,
    \begin{align*}
        Q^{\gamma,\pi_\theta}_{t\up}(O_{1:t\up}, A_{0:t\up}, S_{t\up}) = \bE_{\pi_\theta}\sqbr{R^\gamma_{t\up:\infty} \mid O_{1:t\up}, A_{0:t\up}, S_{t\up}}.
    \end{align*}
    Moreover, by the recursive structure of ASM policies, conditioning on the earlier agent-state history beyond $(O_{1:t\up}, A_{0:t\up}, S_{t\up})$ does not change the conditional distribution of the future trajectory. Hence,
    \begin{align*}
        \bE_{\pi_\theta}\sqbr{R^\gamma_{t\up:\infty} \mid \cF_{t\up}, S_{t\up}, A_{t\up}} = Q^{\gamma,\pi_\theta}_{t\up}(O_{1:t\up}, A_{0:t\up}, S_{t\up}).
    \end{align*}
    Also, note that $\nabla_\theta \log\br{\pi_\theta(S_{t\up}, A_{t\up} \mid S_{t\up-1}, A_{t\up-1}, O_{t\up})}$ is $\sigma(\cF_{t\up}, S_{t\up}, A_{t\up})$-measurable. Since the series in \eqref{pg:ih_1} is absolutely integrable, we may apply the tower property term by term. Thus,
    \begin{align*}
        &\bE_{\pi_\theta}\sqbr{\sum_{t\up=t}^{\infty}{\gamma^{t\up-t} R^\gamma_{t\up:\infty} \nabla_\theta \log\br{\pi_\theta(S_{t\up}, A_{t\up} \mid S_{t\up-1}, A_{t\up-1}, O_{t\up})}} \mid h_t} \\
        &= \bE_{\pi_\theta}\sqbr{\sum_{t\up=t}^{\infty}{\gamma^{t\up-t} \bE_{\pi_\theta}\sqbr{R^\gamma_{t\up:\infty} \nabla_\theta \log\br{\pi_\theta(S_{t\up}, A_{t\up} \mid S_{t\up-1}, A_{t\up-1}, O_{t\up})} \mid \cF_{t\up}, S_{t\up}, A_{t\up}}} \mid h_t} \\
        &= \bE_{\pi_\theta}\sqbr{\sum_{t\up=t}^{\infty}{\gamma^{t\up-t} \bE_{\pi_\theta}\sqbr{R^\gamma_{t\up:\infty} \mid \cF_{t\up}, S_{t\up}, A_{t\up}} \nabla_\theta \log\br{\pi_\theta(S_{t\up}, A_{t\up} \mid S_{t\up-1}, A_{t\up-1}, O_{t\up})}} \mid h_t} \\
        &= \bE_{\pi_\theta}\sqbr{\sum_{t\up=t}^{\infty}{\gamma^{t\up-t} Q^{\gamma,\pi_\theta}_{t\up}(O_{1:t\up}, A_{0:t\up}, S_{t\up}) \nabla_\theta \log\br{\pi_\theta(S_{t\up}, A_{t\up} \mid S_{t\up-1}, A_{t\up-1}, O_{t\up})}} \mid h_t}.
    \end{align*}
    This completes the proof.
\end{proof}

\begin{proof}[\textbf{Proof of Theorem~\ref{thm:ihpg}}]
    Note that $J_\gamma(\theta)$~\eqref{obj:ihdpg} can equivalently be written as follows,
    \begin{align*}
        J_\gamma(\theta) = \bE\sqbr{V^{\gamma,\pi_\theta}_1(O_1,A_0,S_0)} = \sum_{o \in \cO}{\mu(o) V^{\gamma,\pi_\theta}_1(o,a_0,s_0)},
    \end{align*}
    where $\mu(\cdot)$ is the distribution of the initial observation $O_1$. Hence,
    \begin{align*}
        \nabla_\theta{J_\gamma(\theta)} = \sum_{o \in \cO}{\mu(o) \nabla_{\theta} V^{\gamma,\pi_\theta}_1(o,a_0,s_0)}.
    \end{align*}
    The proof follows by using~\eqref{pg:ih_1} and~\eqref{pg:ih_2} from Proposition~\ref{prop:ih_value_grad_t} with $t$ set equal to $1$.
\end{proof}

\subsection{Smoothness of policy gradient objective}
\label{subapp:smooth_discounted}
We make the following standard smoothness assumption on the class of parametrized stationary ASM policies. 
\begin{assumption}\label{assum:smooth_parame_ih}
    The class of parameterized stationary ASM policies, $\{\pi_{\theta}, \theta \in \bR^d\}$ satisfies the following: There exists $G, M > 0$ such that for every $s,\tilde{s} \in \cS$, $a, \tilde{a} \in \cA$, and $o \in \cO$, $$\|\nabla_\theta \log \pi_{\theta}(s,a \mid \tilde{s},\tilde{a},o)\|_2 \le G, \mbox{ and } \|\nabla_\theta^2 \log \pi_{\theta}(s,a \mid \tilde{s},\tilde{a},o)\|_2 \le M.$$
\end{assumption}

Under the above assumption, the parameterized infinite-horizon discounted policy gradient objective~\eqref{obj:ihdpg} is uniformly smooth, which is stated in the following lemma.

\begin{lemma}\label{lem:discounted_smooth}
    Let the ASM policy parameterization satisfy Assumption~\ref{assum:smooth_parame_ih} and $r_{\max}>0$ be such that $|r_t(\cdot)|\le r_{\max}$, for all $t \geq 1$. Then, $J_\gamma(\theta)$ is $\beta_\gamma$-smooth in $\theta$, i.e., $\|\nabla_\theta^2 J_\gamma(\theta)\|_2 \le \beta_\gamma$, where
    \begin{align}
        \beta_\gamma \coloneqq \frac{r_{\max}}{(1-\gamma)^2} \br{M + \frac{G^2(1+\gamma)}{1-\gamma}}.
    \end{align}
\end{lemma}

\begin{proof}[\textbf{Proof of Lemma~\ref{lem:discounted_smooth}}]
    From the policy gradient theorem (Theorem~\ref{thm:ihpg}),
    \begin{align*}
        \nabla_\theta J_\gamma(\theta) =
        \bE_{\pi_\theta}\sqbr{\sum_{t=1}^{\infty}{\gamma^{t-1} Q^{\gamma,\pi_\theta}_t(O_{1:t},A_{0:t},S_t) \nabla_\theta \log \pi_{\theta}(S_t, A_t \mid S_{t-1}, A_{t-1}, O_t)}}.
    \end{align*}
    Differentiating the above, we get,
    \begin{align*}
        \nabla_\theta^2 J_\gamma(\theta) &=
        \bE_{\pi_\theta}\sqbr{\sum_{t=1}^{\infty} \gamma^{t-1} Q^{\gamma,\pi_\theta}_t \nabla_\theta^2 \log \br{\pi_{\theta}}} + \bE_{\pi_\theta}\sqbr{\sum_{t=1}^{\infty}
        \gamma^{t-1} \nabla_\theta \log \br{\pi_{\theta}} (\nabla_\theta Q^{\gamma,\pi_\theta}_t)^\top} \\
        &\quad+ \bE_{\pi_\theta}\sqbr{\sum_{t=1}^{\infty}
        \gamma^{t-1} Q^{\gamma,\pi_\theta}_t \nabla_\theta \log \br{\pi_{\theta}}
        \br{\sum_{t\up=1}^t \nabla_\theta \log \br{\pi_{\theta}}}^\top
       }.
    \end{align*}
    Applying $2$-norm above,
    \begin{align}
        \norm{\nabla_\theta^2 J_\gamma(\theta)}_2 &\leq
        \bE\sqbr{\sum_{t=1}^{\infty} \gamma^{t-1} \abs{Q^{\gamma,\pi_\theta}_t} \norm{\nabla_\theta^2 \log \br{\pi_{\theta}}}_2} + \bE\sqbr{\sum_{t=1}^{\infty} \gamma^{t-1} \norm{\nabla_\theta \log \br{\pi_{\theta}}}_2 \norm{\nabla_\theta Q^{\gamma,\pi_\theta}_t}_2} \notag \\
        &\quad+ \bE\sqbr{\sum_{t=1}^{\infty} \gamma^{t-1} \abs{Q^{\gamma,\pi_\theta}_t} \norm{\nabla_\theta \log \br{\pi_{\theta}}}_2
        \sum_{t\up=1}^t \norm{\nabla_\theta \log \br{\pi_{\theta}}}_2}.\label{bdd:hess2norm_ih}
    \end{align}
    We will produce a uniform upper bound on the r.h.s. of the above inequality. From Assumption~\ref{assum:smooth_parame_ih}, we have $\norm{\nabla_\theta \log\br{\pi_{\theta}}}_2 \leq G$ and $\norm{\nabla^2_\theta \log\br{\pi_{\theta}}}_2 \leq M$. Also, note that $\abs{Q^{\gamma,\pi_\theta}_t} \le \frac{r_{\max}}{1-\gamma}$ which follows from the assumption $|r_t| \leq r_{\max}$. It remains to bound $\norm{\nabla_\theta Q^{\gamma,\pi_\theta}_t(o_{1:t}, a_{0:t}, s_t)}_2$ in order to derive a uniform bound on $\norm{\nabla^2_\theta J_\gamma(\theta)}$. 

    Since
    \begin{align*}
        R^\gamma_{t:\infty}
        = r_t(o_{1:t}, a_{1:t}) + \gamma R^\gamma_{t+1:\infty},
    \end{align*}
    we have
    \begin{align*}
        Q^{\gamma,\pi_\theta}_t(o_{1:t}, a_{0:t}, s_t)
        &= r_t(o_{1:t}, a_{1:t}) \\
        &\quad + \gamma \sum_{o_{t+1} \in \cO}{p_t(o_{t+1} \mid o_{1:t}, a_{1:t}) V^{\gamma,\pi_\theta}_{t+1}(o_{1:t+1}, a_{0:t}, s_t)}.
    \end{align*}
    Differentiating both sides with respect to $\theta$, and using that $r_t$ and $p_t$ do not depend on $\theta$, gives
    \begin{align*}
        \nabla_\theta Q^{\gamma,\pi_\theta}_t(o_{1:t}, a_{0:t}, s_t)
        = \gamma \sum_{o_{t+1} \in \cO}{p_t(o_{t+1} \mid o_{1:t}, a_{1:t}) \nabla_\theta V^{\gamma,\pi_\theta}_{t+1}(o_{1:t+1}, a_{0:t}, s_t)}.
    \end{align*}
    Applying Proposition~\ref{prop:ih_value_grad_t} in the above, we obtain,
    \begin{align}
        &\nabla_\theta Q^{\gamma,\pi_\theta}_t(o_{1:t}, a_{0:t}, s_t) \notag\\
        &= \bE_{\pi_\theta}\sqbr{\sum_{t\up=t+1}^{\infty}{\gamma^{t\up-t} Q^{\gamma,\pi_\theta}_{t\up}(O_{1:t\up}, A_{0:t\up}, S_{t\up}) \nabla_\theta \log\br{\pi_\theta(S_{t\up}, A_{t\up} \mid S_{t\up-1}, A_{t\up-1}, O_{t\up})}} \mid o_{1:t}, a_{0:t}, s_t}.
    \end{align}
    Using Jensen's inequality, Assumption~\ref{assum:smooth_parame_ih} and the bound on $\abs{Q^{\gamma,\pi_\theta}_{t\up}}$, we get
    \begin{align*}
        &\norm{\nabla_\theta Q^{\gamma,\pi_\theta}_t(o_{1:t}, a_{0:t}, s_t)}_2 \\
        &\le \bE_{\pi_\theta}\sqbr{\sum_{t\up=t+1}^{\infty}{\gamma^{t\up-t} \abs{Q^{\gamma,\pi_\theta}_{t\up}(O_{1:t\up}, A_{0:t\up}, S_{t\up})} \norm{\nabla_\theta \log\br{\pi_\theta(S_{t\up}, A_{t\up} \mid S_{t\up-1}, A_{t\up-1}, O_{t\up})}}_2} \mid o_{1:t}, a_{0:t}, s_t} \\
        &\le \sum_{t\up=t+1}^{\infty}{\gamma^{t\up-t} \cdot \frac{r_{\max}}{1-\gamma} \cdot G} \\
        &= \frac{\gamma G r_{\max}}{(1-\gamma)^2}.
    \end{align*}
    
    Substituting the bounds on $\norm{\nabla_\theta \log\br{\pi_{\theta}}}_2$ and $\norm{\nabla^2_\theta \log\br{\pi_{\theta}}}_2$, $\abs{Q^{\gamma,\pi_\theta}_t}$ and $\norm{\nabla_\theta Q^{\gamma,\pi_\theta}_t}_2$ in \eqref{bdd:hess2norm_ih}, we obtain,
    \begin{align*}
        \|\nabla_\theta^2 J_\gamma(\theta)\|_2 &\leq \frac{r_{\max} M}{1-\gamma} \sum_{t=1}^{\infty}{\gamma^{t-1}} + \frac{r_{\max} G^2 \gamma}{(1-\gamma)^2} \sum_{t=1}^{\infty}{\gamma^{t-1}} + \frac{r_{\max} G^2}{1-\gamma} \sum_{t=1}^{\infty}{t\, \gamma^{t-1}} \\
        &\leq \frac{r_{\max}M}{(1-\gamma)^2} + \frac{r_{\max} G^2\gamma}{(1-\gamma)^3} + \frac{r_{\max} G^2}{(1-\gamma)^3} = \beta_\gamma,
    \end{align*}
    which completes the proof.
\end{proof}

\subsection{Convergence results}
\label{subapp:proof_conv_discounted}
Next, we establish finite-time expected and high-probability convergence rates for the time-averaged gradient norm of ASMPG~(Algorithm~\ref{algo:ihdnmpg}) iterates with constant stepsize, along with finite-time and almost sure convergence results under decreasing stepsizes.
\begin{theorem}[Finite-time convergence of discounted ASM policy gradient]
    \label{thm:conv_ihpg}
    Consider the iterates $\{\theta_k\}_{k\ge 1}$ generated by Algorithm~\ref{algo:ihdnmpg}. Suppose that the stationary ASM policy parameterization satisfies Assumption~\ref{assum:smooth_parame_ih}. Let $\Delta_1 := \sup_{\theta\in\bR^d} J_\gamma(\theta)-J_\gamma(\theta_1)$, and let $\beta_\gamma$ be the smoothness constant in Lemma~\ref{lem:discounted_smooth}. Let $C_\gamma := \frac{2 r_{\max}G}{(1-\gamma)^2}$.
    Then, the following statements hold.
    \begin{enumerate}[a)]
        \item If $\alpha_k = \min\left\{\frac{1}{\beta_\gamma}, \frac{2}{C_\gamma}\sqrt{\frac{2\Delta_1}{\beta_\gamma K}}\right\},~ \forall k$, then
        \begin{align*}
            \frac{1}{K}\sum_{k=1}^K \mathbb{E}\!\left[\|\nabla_\theta J_\gamma(\theta_k)\|_2^2\right] \leq \frac{2\beta_\gamma\Delta_1}{K} + C_\gamma\sqrt{\frac{\beta_\gamma\Delta_1}{2K}}.
        \end{align*}
        \item With probability at least $1-\delta$, the iterates satisfy the following bounds:
        \begin{enumerate}[(i)]
            \item If $\alpha_k = \min\left\{\frac{1}{\beta_\gamma}, \frac{1}{C_\gamma}\sqrt{\frac{\Delta_1}{\beta_\gamma K}}\right\},~\forall k$, then
            \begin{align*}
                \frac{1}{K}\sum_{k=1}^{K}{\norm{\nabla_\theta J_\gamma(\theta_k)}_2^2} \leq \frac{2 \beta_\gamma \Delta_1}{K} + 5 C_\gamma \sqrt{\frac{\beta_\gamma\Delta_1}{K}} + \frac{12 C_\gamma^2 \log(1/\delta)}{K}.
            \end{align*}
            \item If $\alpha_k=\frac{1}{\beta_\gamma\sqrt{k}},~\forall k$, then
            \begin{align*}
                \frac{1}{K}\sum_{k=1}^{K}{\norm{\nabla_\theta J_\gamma(\theta_k)}_2^2} \leq \frac{2 \beta_\gamma \Delta_1 + 3 C_\gamma^2 (1+\log K) + 12 C_\gamma^2 \log(1/\delta)}{\sqrt{K}}.
            \end{align*}
        \end{enumerate}
        \item If $\{\alpha_k\}$ is any positive sequence satisfying $\sum_{k}{\alpha_k} = \infty$ and $\sum_{k}{\alpha^2_k} < \infty$, then with probability $1$, $\norm{\nabla_\theta J_\gamma(\theta_k)} \to 0$ as $k \to \infty$.
    \end{enumerate}
\end{theorem}

\begin{proof} 
    We first derive a bound on $\norm{\widehat{g}_{\gamma,k}}_2$. Since $|r_t|\le r_{\max}$ and by Assumption~\ref{assum:smooth_parame_ih}, we have almost surely
    \begin{align*}
        \norm{\widehat{g}_{\gamma,k}}_2 &\leq
        \sum_{t=1}^{\infty}{\abs{\sum_{t\up=t}^{\infty}{\gamma^{t\up-1} r_{t\up}(O^k_{1:t\up},A^k_{1:t\up})}} \norm{\nabla_\theta \log\br{\pi_{\theta_k}(S^k_t,A^k_t \mid S^k_{t-1},A^k_{t-1},O^k_t)}}_2} \\
        &\leq \sum_{t=1}^{\infty}{\gamma^{t-1} \frac{r_{\max} G}{1-\gamma}} \\
        &= \frac{r_{\max} G}{(1-\gamma)^2} = C_\gamma/2.
    \end{align*}
    The above bound implies a bounded-variance (and sub-Gaussian) noise condition with parameter $C_\gamma$. From here the proof follows from the proof of Theorem~\ref{thm:conv_fhpg}, by substituting $C_H$ with $C_\gamma$ and $\beta$ with $\beta_\gamma$.
\end{proof}

\section{Condition for ASM policies to be optimal}
\label{app:asd_optimality}

In this section, we identify a sufficient condition under which the restriction to the class of ASM policies while optimizing for the expected cumulative reward incurs no loss of optimality. Equivalently, there is an ASM policy that is optimal. This condition is that the agent state generated by the agent state dynamics is a sufficient statistic of the history for both rewards and future observations.

\begin{definition}[Ideal agent state dynamics]
    \label{def:ideal_asd}
    Consider an agent state dynamics $\nu=\{\nu_t\}_{t=1}^H$ and let the agent state process
    $\{S_t\}_{t=1}^H$ be generated recursively as
    \begin{align*}
        S_t \sim \nu_t(\cdot \mid S_{t-1},A_{t-1},O_t),\quad t\in[H],
    \end{align*}
    with $(S_0,A_0)=(s_0,a_0)$, where $(s_0,a_0)$ is a designated initial
    state-action pair. We say that $\nu$ is an \emph{ideal agent state dynamics} if, for every $t\in[H]$, there exists $\bar r_t:S \times A \to \bR$ such that 
    \begin{align*}
        r_t(O_{1:t},A_{1:t}) = \bar{r}_t(S_t,A_t),
    \end{align*}
    and for every $t\in[H-1]$, there exists $\bar{p}_t: S \times A \to \Delta_\cO$ such that
    \begin{align*}
        p_t(\cdot\mid O_{1:t},A_{1:t}) = \bar{p}_t(\cdot\mid S_t,A_t).
    \end{align*}
\end{definition}

Definition~\ref{def:ideal_asd} states that, given the current action, the agent state $S_t$ contains all information from the history that is relevant for the current reward and for predicting the next observation. Thus, once $S_t$ is known, the rest of the past observations and actions can be discarded.

The following lemma shows that once an ideal agent state dynamics is known, it suffices to optimize over only those decision rules that use only the current agent state. A closely related result appears in \citep{subramanian2022approximate}.

\begin{lemma}[Optimality of ASM policies under an ideal agent state dynamics]
    \label{lem:asd_markov_optimal}
    Suppose there exists an ideal agent state dynamics
    $\nu\ust=\{\nu_t\ust\}_{t=1}^H$. Then, there is no loss of optimality in restricting attention to the class of ASM policies that use $\nu\ust$ while maximizing the expected cumulative reward. In
    particular,
    \begin{align*}
        \max_{\phi} \bE_{\phi\circ \nu\ust,\mu}\sqbr{R_{1:H}} = \max_{\pi\in \Pi_{\mathrm{HR}}} \bE_{\pi,\mu}\sqbr{R_{1:H}},
    \end{align*}
    where $\phi=\{\phi_t:\cS\to\Delta_{\cA}\}_{t=1}^H$ ranges over
    agent-state-dependent Markov decision rules.
\end{lemma}

\begin{proof}
    Every ASM policy $\phi\circ\nu\ust$ induces a history-dependent
    policy. Therefore,
    \begin{align}
        \max_{\phi} \bE_{\phi\circ \nu\ust,\mu}\sqbr{R_{1:H}} \leq \max_{\pi\in \Pi_{\mathrm{HR}}} \bE_{\pi,\mu}\sqbr{R_{1:H}}. \label{ineq:leq}
    \end{align}
    It remains to prove the reverse inequality. Fix an arbitrary history-dependent policy $\pi=\{\pi_t\}_{t=1}^H\in\Pi_{\mathrm{HR}}$. Now consider the ideal agent state dynamics $\nu\ust$ which uses
    $$S_t \sim \nu_t\ust(\cdot\mid S_{t-1},A_{t-1},O_t).$$
    
    We will now construct an ASM policy $\phi^\pi\circ\nu\ust$ such that it will induce the same probability as that done by $\pi \circ\nu\ust$. Let $\bP^{\pi,\nu\ust}$ denote the probability measure induced by policy $\pi$ and agent state dynamics $\nu\ust$. For each $t\in[H]$ and $s\in\cS$, define
    \begin{align*}
        \phi_t^\pi(a\mid s) := \bP^{\pi,\nu\ust}(A_t=a\mid S_t=s),
    \end{align*}
    whenever $\bP^{\pi,\nu\ust}(S_t=s)>0$. If $\bP^{\pi,\nu\ust}(S_t=s) = 0$, define $\phi_t^\pi(\cdot\mid s)$ arbitrarily. We will show that the following holds:
    \begin{align}
        \bP^{\pi,\nu\ust}(S_t=s,A_t=a) = \bP^{\phi^\pi\circ\nu\ust}(S_t=s,A_t=a), \quad \forall (s,a)\in\cS\times\cA, \forall t \in [H]. \label{eq:ind_hyp}
    \end{align}
    We prove the claim by induction on $t$.
    
    The base case $t=1$: the distribution of $S_1$ is the same under both policies because $S_1$ is generated from the same initial pair $(S_0,A_0)=(s_0,a_0)$, the same initial observation distribution, and the same kernel $\nu_1\ust$. Moreover, by construction,
    \begin{align*}
        \bP^{\phi^\pi\circ\nu\ust}(A_1=a\mid S_1=s) = \phi_1^\pi(a\mid s) = \bP^{\pi,\nu\ust}(A_1=a\mid S_1=s)
    \end{align*}
    on the support of $S_1$. Hence, the claim holds for $t=1$.
    
    Now suppose that the \eqref{eq:ind_hyp} holds at time $t$. Since $\nu\ust$ is ideal, the
    conditional law of $O_{t+1}$ depends on the history only via $(S_t,A_t)$:
    \begin{align*}
        \bP(O_{t+1}=o\mid O_{1:t},A_{1:t}) = \bar{p}_t(o\mid S_t,A_t).
    \end{align*}
    Therefore, under the fixed agent state dynamics $\nu\ust$, the one-step transition kernel of the agent state is
    \begin{align*}
        \bar{P}_t(s'\mid s,a) := \sum_{o\in\cO} \bar{p}_t(o\mid s,a) \nu_{t+1}\ust(s'\mid s,a,o).
    \end{align*}
    Consequently,
    \begin{align*}
        \bP^{\pi,\nu\ust}(S_{t+1}=s') &= \sum_{s,a}{\bP^{\pi,\nu\ust}(S_t=s,A_t=a) \bar{P}_t(s'\mid s,a)} \\
        &= \sum_{s,a}{\bP^{\phi^\pi\circ\nu\ust}(S_t=s,A_t=a)\,
        \bar{P}_t(s'\mid s,a)} \\
        &= \bP^{\phi^\pi\circ\nu\ust}(S_{t+1}=s'),
    \end{align*}
    where the second equality uses the induction hypothesis. Finally, by the
    definition of $\phi_{t+1}^\pi$,
    \begin{align*}
        \bP^{\phi^\pi\circ\nu\ust}(A_{t+1}=a'\mid S_{t+1}=s') = \phi_{t+1}^\pi(a'\mid s') = \bP^{\pi,\nu\ust}(A_{t+1}=a'\mid S_{t+1}=s')
    \end{align*}
    on the support of $S_{t+1}$. Hence, the joint distribution of
    $(S_{t+1},A_{t+1})$ also agrees under the two policies. This completes the
    induction.
    
    Using the reward condition in Definition~\ref{def:ideal_asd}, we obtain
    \begin{align*}
        \bE_{\pi,\mu}\sqbr{R_{1:H}}
        &=
        \bE_{\pi,\nu\ust}
        \left[
            \sum_{t=1}^H r_t(O_{1:t},A_{1:t})
        \right] \\
        &= \sum_{t=1}^{H}{\bE_{\pi,\nu\ust}\sqbr{\bar{r}_t(S_t,A_t)}} \\
        &= \sum_{t=1}^{H}{\bE_{\phi^\pi\circ\nu\ust}\sqbr{\bar{r}_t(S_t,A_t)}} \\
        &= \bE_{\phi^\pi \circ \nu\ust,\mu}\sqbr{R_{1:H}}.
    \end{align*}
    Thus, for every history-dependent policy $\pi\in\Pi_{\mathrm{HR}}$, there
    exists an ASM policy $\phi^\pi\circ\nu\ust$ with the same value.
    Therefore,
    \begin{align}
        \max_{\pi\in\Pi_{\mathrm{HR}}} \bE_{\pi,\mu}\sqbr{R_{1:H}} \leq \max_{\phi} \bE_{\phi\circ \nu\ust,\mu}\sqbr{R_{1:H}}. \label{ineq:geq}
    \end{align}
    Combining the two \eqref{ineq:leq} and \eqref{ineq:geq} yields
    \begin{align*}
        \max_{\phi} \bE_{\phi\circ \nu\ust,\mu}\sqbr{R_{1:H}} = \max_{\pi\in \Pi_{\mathrm{HR}}} \bE_{\pi,\mu}\sqbr{R_{1:H}}.
    \end{align*}
\end{proof}

Lemma~\ref{lem:asd_markov_optimal} implies that an ideal agent state dynamics reduces the optimization problem from the class of all history-dependent policies to the class of ASM policies. In particular, if $S_t$ takes values in a low-dimensional space, then the agent state dynamics acts as an efficient compressor of the trajectory.

\begin{remark}[Examples of Ideal agent state dynamics]
    A canonical example is a POMDP, where the hidden environment state is not observed directly. However, the belief state, i.e., the conditional distribution of the hidden state given the observation-action history, is an information state. The belief-update recursion is therefore an ideal agent state dynamics, and Lemma~\ref{lem:asd_markov_optimal} recovers the classical fact that it suffices to optimize over policies that depend only on the belief state.
    
    Another example is a velocity-only CartPole system. Suppose the agent observes only the linear and angular velocities, while the initial positions are known and the observations are noise-free. The position variables can then be reconstructed recursively by integrating the velocity over time. In this case, the agent state dynamics is a deterministic integrator that reconstructs a sufficient state from the observation history, and the resulting agent state is sufficient for predicting future observations and rewards.
\end{remark}
\section{Properties of softmax ASM policies}
\label{app:aux}
Let $\cS$, $\cA$, and $\cO$ be finite sets and let $\theta \in \bR^{\abs{\cS}^2 \abs{\cA}^2\abs{\cO}H}$. The softmax ASM policy $\pi_\theta$, parameterized by $\theta$, is defined as,
\begin{align*}
    \pi_{\theta,t}(s,a\mid \tilde s,\tilde a,o) = \frac{\exp(\theta(\tilde s,\tilde a,o,s,a,t))}
    {\sum_{(s',a')\in \cS\times \cA}\exp(\theta(\tilde s,\tilde a,o,s',a',t))},
\end{align*}
for every $(\tilde{s},\tilde{a},o)\in \cS \times \cA \times \cO$, $(s,a) \in \cS \times \cA$ and $t \in [H]$.~The class of softmax ASM policies satisfies the uniform score and Hessian bounds required in Assumption~\ref{assum:smooth_parame_fh}, as shown in the following lemma.

\begin{lemma}[Lipschitz continuity and smoothness]
\label{lem:softmax_asm}
    Consider the class of softmax ASM policies $\{\pi_\theta : \theta \in \bR^{\abs{\cS}^2 \abs{\cA}^2\abs{\cO}H}\}$. For every
    $t\in[H]$,
    \begin{align*}
        \norm{\nabla_{\theta} \log \pi_{\theta,t}(s,a\mid \tilde s,\tilde a,o)}_2 \leq \sqrt{2},
    \end{align*}
    and
    \begin{align*}
        \norm{\nabla_{\theta}^2 \log\br{\pi_{\theta,t}(s,a\mid \tilde s,\tilde a,o)}}_2 \leq 1.
    \end{align*}
    Consequently, the softmax parameterized ASM policy class satisfies
    Assumption~\ref{assum:smooth_parame_fh} with $G=\sqrt{2}$ and $M=1$.
\end{lemma}

\begin{proof}
    Fix $t\in[H]$ and an input triple
    $x=(\tilde s,\tilde a,o)\in\cS\times\cA\times\cO$.
    For notational convenience, write $y=(s,a)\in\cS\times\cA$ and
    \begin{align*}
        p_\theta(y\mid x;t) := \pi_{\theta,t}(s,a\mid \tilde s,\tilde a,o).
    \end{align*}
    The parameter block corresponding to the input $x$ is
    $\{\theta(x,y',t) : y' \in \cS \times \cA\}$.
    
    For any $y'\in\cS\times\cA$, the derivative of the log-softmax is
    \begin{align*}
        \frac{\partial}{\partial \theta(x,y',t)}\log\br{p_\theta(y\mid x)} = \mathbf 1\{y'=y\} - p_\theta(y'\mid x;t).
    \end{align*}
    For parameter coordinates corresponding to input triples $x'\neq x$, the
    derivative is zero. Hence
    \begin{align*}
        \norm{\nabla_\theta \log\br{p_\theta(y\mid x;t)}}_2^2 &=
        \sum_{y' \in \cS \times \cA} \br{\mathbf 1\{y'=y\} - p_\theta(y'\mid x;t)}^2  \\
        &= (1 - p_\theta(y\mid x;t))^2 + \sum_{y'\neq y} p_\theta(y'\mid x;t)^2  \\
        &\leq (1 - p_\theta(y\mid x;t))^2 + \br{\sum_{y'\neq y}{p_\theta(y'\mid x;t)}}^2 \\
        &= 2(1-p_\theta(y\mid x))^2 \leq 2.
    \end{align*}
    Therefore,
    \begin{align*}
        \norm{\nabla_\theta \log\br{p_\theta(y\mid x; t)}}_2 \leq \sqrt{2}.
    \end{align*}
    Next, the Hessian with respect to the active parameter block is
    \begin{align*}
        \nabla_{\theta(x,\cdot,t)}^2 \log\br{p_\theta(y\mid x;t)} = -\br{\operatorname{diag}(p_\theta(\cdot\mid x;t)) - p_\theta(\cdot \mid x; t) p_\theta(\cdot \mid x; t)^\top}.
    \end{align*}
    All mixed derivatives involving parameter blocks $x'\neq x$ are zero. Thus, it suffices to bound the spectral norm of
    \begin{align*}
        \operatorname{diag}(p_\theta(\cdot\mid x; t)) - p_\theta(\cdot\mid x; t)\, p_\theta(\cdot\mid x; t)^\top.
    \end{align*}
    For any vector $v$ with $\norm{v}_2=1$,
    \begin{align*}
        v^\top \br{\operatorname{diag}(p_\theta(\cdot\mid x; t)) - p_\theta(\cdot\mid x; t)\, p_\theta(\cdot\mid x; t)^\top} v &= \sum_{y'} p_\theta(y'\mid x)v_{y'}^2 - \br{\sum_{y'}p_\theta(y'\mid x)v_{y'}}^2  \\
        &\leq \sum_{y'}p_\theta(y'\mid x)v_{y'}^2  \\
        &\leq \sum_{y'}v_{y'}^2 = 1.
    \end{align*}
    Since the above matrix is positive semidefinite, its spectral norm is at most $1$. Therefore,
    \begin{align*}
        \norm{\nabla_{\theta}^2 \log\br{\pi_{\theta,t}(s,a\mid \tilde s,\tilde a,o)}}_2 \leq 1.
    \end{align*}
    This proves the claim.
\end{proof}

\section{Simulation details}
\label{app:simulation_details}
\subsection{Environment description}
In this section, we describe the simulation environments in greater detail.

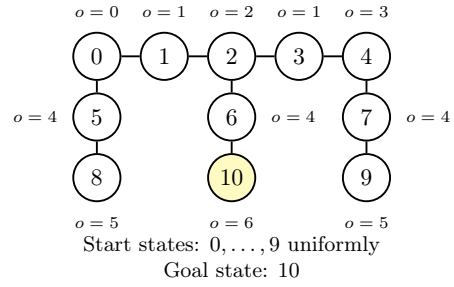
\begin{wrapfigure}{!t}{0.4\textwidth}
    \centering
    \resizebox{\linewidth}{!}{\begin{tikzpicture}[x=1cm,y=0.9cm]

% Top row
\node[state] (s0) at (0,2) {0};
\node[state] (s1) at (1,2) {1};
\node[state] (s2) at (2,2) {2};
\node[state] (s3) at (3,2) {3};
\node[state] (s4) at (4,2) {4};

% Middle row
\node[state] (s5) at (0,1) {5};
\node[state] (s6) at (2,1) {6};
\node[state] (s7) at (4,1) {7};

% Bottom row
\node[state] (s8) at (0,0) {8};
\node[goalstate] (s10) at (2,0) {10};
\node[state] (s9) at (4,0) {9};

% Horizontal edges (top row)
\draw[thick] (s0) -- (s1);
\draw[thick] (s1) -- (s2);
\draw[thick] (s2) -- (s3);
\draw[thick] (s3) -- (s4);

% Vertical edges
\draw[thick] (s0) -- (s5);
\draw[thick] (s2) -- (s6);
\draw[thick] (s4) -- (s7);

\draw[thick] (s5) -- (s8);
\draw[thick] (s6) -- (s10);
\draw[thick] (s7) -- (s9);

% Observation labels
\node[obslabel, above=1mm of s0] {$o=0$};
\node[obslabel, above=1mm of s1] {$o=1$};
\node[obslabel, above=1mm of s2] {$o=2$};
\node[obslabel, above=1mm of s3] {$o=1$};
\node[obslabel, above=1mm of s4] {$o=3$};

\node[obslabel, left=1mm of s5] {$o=4$};
\node[obslabel, right=1mm of s6] {$o=4$};
\node[obslabel, right=1mm of s7] {$o=4$};

\node[obslabel, below=1mm of s8] {$o=5$};
\node[obslabel, below=1mm of s9] {$o=5$};
\node[obslabel, below=1mm of s10] {$o=6$};

% Start note
\node[align=center, font=\small] at (2,-1.3)
{Start states: $0,\dots,9$ uniformly \\ Goal state: $10$};

\end{tikzpicture}}
    \caption{CheeseMaze.}
    \label{fig:cheesemaze}
\end{wrapfigure}

\textbf{CheeseMaze.} CheeseMaze is a partially observed navigation problem based on the maze example in~\citet{mccallum1993overcoming}.~The environment has $11$ latent states, labeled $0,\ldots,10$, where state $10$ is the terminal goal state. The agent observes only $7$ distinct observation symbols. Thus, multiple maze locations are observationally aliased as shown in Figure~\ref{fig:cheesemaze}.~At the beginning of each episode, the agent starts uniformly at random from one of the ten non-terminal states. At each step, it chooses one of four movement actions: north, south, east, or west. The agent receives a reward of $1$ upon reaching the goal state and $0$ otherwise. The episode terminates when the goal is reached.

The key difficulty is that the same observation can correspond to different latent maze locations requiring different actions. Hence, a policy that depends solely on the current observation cannot be optimal. Successful navigation requires the agent to infer its latent location from the history of previous observations and actions.

\begin{wrapfigure}{!t}{0.4\textwidth}
    \centering
    \resizebox{\linewidth}{!}{\begin{tikzpicture}[x=1cm,y=1cm]

% Draw outer accessible cells (7 x 4)
\foreach \x in {0,...,6} {
    \foreach \y in {0,...,3} {
        \draw[gray!50] (\x,\y) rectangle ++(1,1);
    }
}

% Blocked cells
\fill[black!25] (1,1) rectangle ++(1,1);
\fill[black!25] (2,1) rectangle ++(1,1);
\fill[black!25] (1,2) rectangle ++(1,1);
\fill[black!25] (2,2) rectangle ++(1,1);

\fill[black!25] (4,1) rectangle ++(1,1);
\fill[black!25] (5,1) rectangle ++(1,1);
\fill[black!25] (4,2) rectangle ++(1,1);
\fill[black!25] (5,2) rectangle ++(1,1);

% Thick boundary around blocked cells
\draw[thick] (1,1) rectangle (3,3);
\draw[thick] (4,1) rectangle (6,3);

% Start positions (corners)
\node[font=\scriptsize] at (0.5,3.5) {$S$};
\node[font=\scriptsize] at (6.5,3.5) {$S$};
\node[font=\scriptsize] at (0.5,0.5) {$S$};
\node[font=\scriptsize] at (6.5,0.5) {$S$};

% Goal
\fill[yellow!40] (3,2) rectangle ++(1,1);
\draw[thick] (3,2) rectangle ++(1,1);
\node at (3.5,2.5) {$G$};

% Labels
\node[align=center, font=\small] at (3.5,-0.8)
{$s$ denotes possible start locations \\ Observation: local wall pattern in N/E/S/W \\ Actions: North, East, South, West};

\end{tikzpicture}}
    \caption{HallwayNavigation.}
    \label{fig:hallway_nav}
\end{wrapfigure}

\textbf{HallwayNavigation.} HallwayNavigation is a deterministic, partially observed maze-navigation task introduced by \citet{mccallum1995instance}. The latent state is the agent's position in a $7\times 4$ grid containing two interior blocked regions~(See Figure~\ref{fig:hallway_nav}). The agent starts uniformly at one of the four corners and must reach the central goal cell. The available actions are north, east, south, and west. If an action would move the agent into a wall or a blocked cell, the agent remains in the same place. The agent does not observe its absolute grid position. Instead, it observes whether there is a wall immediately to the north, east, south, and west, giving up to $2^4=16$ possible observations. Reaching the goal gives a reward of $5.0$, attempting to move into a wall gives a reward of $-1.0$, and all other transitions give a reward of $-0.1$.

Note that different grid cells can have identical local wall configurations. Thus, the agent needs to take different actions to reach the goal from cells with the same observation. Therefore, the current observation is not a sufficient statistic for control, and the agent must use history to disambiguate its position.

\textbf{HealthcareTreatment.} Reinforcement learning has been applied to treatment planning and dynamic treatment regimes, typically using Markovian state-based formulations~\citep{liu2017deep,choudharyicu,Choi2024}. In realistic treatment settings, however, improvement\slash deterioration of health can depend on intervention history, since repeated therapy may accumulate toxicity and induce resistance, thereby altering the effect of future treatment~\citep{strobl2024modulate}. Motivated by this, we consider a sequential treatment-planning environment in which the agent observes the patient's current health state and chooses one of three interventions: no treatment, mild treatment, or aggressive treatment. These interventions have different efficacy-cost trade-offs: aggressive treatment has the largest immediate effect but also the highest cost. The environment also maintains two hidden variables: an accumulated toxicity burden and a treatment-resistance variable. Toxicity refers to the cumulative side effects of previous interventions. It increases after treatment, especially after aggressive treatment, and decays gradually over time. Treatment resistance represents reduced effectiveness of future interventions after repeated aggressive treatment. It increases when aggressive treatment is used and also decays gradually over time.

At each step, the patient's health is updated by combining three effects. First, the chosen intervention provides a treatment benefit. Second, this benefit is attenuated by the current hidden resistance level, so repeated aggressive treatment can make future treatments less effective. Third, accumulated hidden toxicity contributes a negative effect on health. Thus, aggressive treatment can be useful in the short term because it has the largest nominal treatment effect, but excessive aggressive treatment may be harmful in the long term because it increases both toxicity and resistance. This environment is non-Markovian because the agent observes only the current health state, not the hidden toxicity or resistance variables.

The reward is the post-treatment health minus the intervention cost. The episode terminates if the patient's health crosses either a recovery or failure threshold. In our implementation, crossing the upper recovery threshold $h=2.0$ gives a terminal reward of $+50$, while crossing the lower failure threshold $h=-1.0$ gives a terminal penalty of $-50$. Episodes start from health $h=0.5$ with zero toxicity and zero resistance.

\textbf{MachineRepair.} MachineRepair is a maintenance-control environment~\citep{eckles1968optimum} for a single deteriorating machine. Maintenance models often go beyond Markovian dynamics by allowing the degradation of the machine to depend on accumulated age or wear, and by allowing repairs to be imperfect rather than as-good-as-new \citep{pham1996imperfect,giorgio2011age}. Motivated by this, we consider a machine maintenance environment in which the agent observes only a coarse machine condition, healthy or degraded, and chooses between continuing operation and performing repair. In addition to the observed condition, the simulator maintains an unobserved wear variable that evolves with the usage and repair history: wear increases during operation and is only partially reduced by repair. The transition probabilities depend on this hidden wear level, so that a healthy machine becomes more likely to degrade as wear increases, a degraded machine becomes more likely to remain degraded when wear is high, and repair becomes less effective as accumulated wear grows. Consequently, two machines with the same observed condition can exhibit different transition behavior depending on their past usage and maintenance history. Continuing to operate the machine in the healthy state yields a positive reward, whereas in the degraded state, it yields a negative reward; repairs incur a fixed unit cost. This provides a simple maintenance benchmark with history-dependent dynamics in the spirit of age- and wear-dependent degradation and partially observed deterioration models~\citep{pang2023condition}.

\textbf{VelocityOnlyCartPole.} 
VelocityOnlyCartPole~\citep{morad2023popgym} is a partially observed variant of the classic CartPole control problem \citep{barto1983neuronlike,brockman2016openai}. In the standard CartPole system, the state vector consists of the cart position, cart velocity, pole angle, and pole angular velocity. The agent applies either a leftward or rightward force to the cart, and the objective is to keep the pole balanced while keeping the cart within the allowed track.

In the velocity-only version, the agent observes only velocity information: the cart velocity and the pole angular velocity. The cart position and pole angle are hidden. The agent receives a positive reward at each time step for keeping the pole upright and the system within its stability limits, and the episode terminates if the pole angle becomes too large, the cart leaves the allowed track, or the maximum episode length is reached.

Two trajectories with the same velocity components may require different actions depending on the cart and pole positions. The task is therefore non-Markovian, and effective control requires a history of past observations and actions. We note that the position variables can be reconstructed iteratively by integrating the corresponding velocity component, which aligns with the recursive structure of the agent state dynamics.

\subsection{Experiment details}
Table~\ref{tab:params} lists the number of underlying agent states and the dimensions of the hidden layers we used for ASMPG in each simulation environment. We use the following $10$ randomly generated integer values as seeds for $10$ runs: $1952$, $5235$, $8234$, $8386$, $1682$, $3659$, $9848$, $9119$, $6892$, and $9381$.

\begin{table}[t]
    \centering
    \begin{tabular}{|c|c|c|c|c|c|}
        \hline
        Environment & CheeseMaze & Hallway- & Healthcare- & MachineRepair & Velocity- \\
         &  & Navigation & Treatment &  & OnlyCartPole \\
        \hline
        $\abs{\cS}$ & $8$ & $8$ & $3$ & $8$ & $4$ \\
        \hline
        $d_h$ & $128$ & $64$ & $64$ & $64$ & $128$ \\
        \hline
    \end{tabular}
    \caption{$\abs{\cS}$ and $d_h$ values used for each simulation environment.}
    \label{tab:params}
\end{table}

\paragraph{Computational resources.} All experiments were run on a single workstation with an AMD Ryzen Threadripper PRO 5965WX 24-core CPU with 48 hardware threads, 125 GiB of RAM, and one NVIDIA RTX A6000 GPU with 49,140 MiB of GPU memory. The machine used NVIDIA driver version 570.211.01 and CUDA version 12.8. Across five environments and 10 seeds, the reported ASMPG experiments required approximately 25 GPU-hours; baseline runs required approximately 50 GPU-hours in total.

\bibliography{refs}

\end{document}